\documentclass{article}

\usepackage{microtype}
\usepackage{graphicx}

\usepackage{booktabs} 

\usepackage{microtype}
\usepackage{graphicx}
\usepackage{subcaption}
\usepackage{booktabs} 
\usepackage{xcolor}

\usepackage[pagebackref=true,breaklinks=true,colorlinks,bookmarks=false]{hyperref}
 \hypersetup{
 linkcolor=red,
 filecolor=red,
 citecolor=green,      
 urlcolor=cyan,
}

\usepackage{pgfplots}
\pgfplotsset{compat = newest}
\usepackage{multirow}

\usepackage{framed}
\usepackage{tcolorbox}

\usepackage{amsmath}
\usepackage{amssymb}
\usepackage{mathrsfs}
\usepackage{mathtools}
\usepackage{amsthm}
\usepackage[ruled,vlined]{algorithm2e}
\usepackage{algorithmic}
\usepackage{listings}
\usepackage{makecell}
\usepackage{colortbl}
\usepackage{color}
\usepackage{wrapfig}
\usepackage{cancel}
\usepackage{soul,xcolor}
\usepackage{pifont}
\usepackage{tikz}
\usetikzlibrary{shapes, positioning, arrows.meta, calc, decorations.pathmorphing, quotes}

\usepackage[capitalize,noabbrev]{cleveref}

\theoremstyle{plain}
\newtheorem{theorem}{Theorem}[section]
\newtheorem{proposition}[theorem]{Proposition}
\newtheorem{lemma}[theorem]{Lemma}
\newtheorem{corollary}[theorem]{Corollary}
\theoremstyle{definition}
\newtheorem{definition}[theorem]{Definition}

\theoremstyle{remark}

\newcommand{\alink}[1]{\href{#1}{paper-link}}

\usepackage{enumitem}
\usepackage{ amssymb }

\definecolor{citecolor}{HTML}{0071BC}
\definecolor{linkcolor}{HTML}{ED1C24}
\hypersetup{colorlinks=true, linkcolor=linkcolor, citecolor=citecolor,urlcolor=black}

\definecolor{commentcolor}{RGB}{110,154,155}   

\usepackage{amsmath,amsfonts,bm}

\def\eqref#1{equation~\ref{#1}}

\def\1{\bm{1}}

\DeclareMathAlphabet{\mathsfit}{\encodingdefault}{\sfdefault}{m}{sl}
\SetMathAlphabet{\mathsfit}{bold}{\encodingdefault}{\sfdefault}{bx}{n}

\usepackage{hyperref}

\usepackage[accepted]{icml2024}

\usepackage{amsmath}
\usepackage{amssymb}
\usepackage{mathtools}
\usepackage{amsthm}

\usepackage[capitalize,noabbrev]{cleveref}

\usepackage[textsize=tiny]{todonotes}

\usepackage[capitalize,noabbrev]{cleveref}

\icmltitlerunning{Information Flow in Self-Supervised Learning}

\begin{document}

\twocolumn[
\icmltitle{Information Flow in Self-Supervised Learning}




\begin{icmlauthorlist}

\icmlauthor{Zhiquan Tan}{thumath}
\icmlauthor{Jingqin Yang}{iiis}
\icmlauthor{Weiran Huang}{sjtu,shailab}
\icmlauthor{Yang Yuan}{iiis,shailab,qizhi}
\icmlauthor{Yifan Zhang}{iiis}
\end{icmlauthorlist}

\icmlaffiliation{iiis}{IIIS, Tsinghua University, Beijing, China}
\icmlaffiliation{thumath}{Department of Mathematical Sciences, Tsinghua University, Beijing, China}
\icmlaffiliation{shailab}{Shanghai AI Laboratory, Shanghai, China}
\icmlaffiliation{qizhi}{Shanghai Qizhi Institute, Shanghai, China}
\icmlaffiliation{sjtu}{MIFA Lab, Qing Yuan Research Institute, SEIEE, Shanghai Jiao Tong University, Shanghai, China}

\icmlcorrespondingauthor{Yang Yuan}{yuanyang@tsinghua.edu.cn}
\icmlcorrespondingauthor{Yifan Zhang}{zhangyif21@mails.tsinghua.edu.cn}
\icmlkeywords{Machine Learning, ICML}

\vskip 0.3in
]



\printAffiliationsAndNotice{}  

\begin{abstract}
In this paper, we conduct a comprehensive analysis of two dual-branch (Siamese architecture) self-supervised learning approaches, namely Barlow Twins and spectral contrastive learning, through the lens of matrix mutual information. We prove that the loss functions of these methods implicitly optimize both matrix mutual information and matrix joint entropy. This insight prompts us to further explore the category of single-branch algorithms, specifically MAE and U-MAE, 
for which mutual information and joint entropy become the entropy. 
Building on this intuition, we introduce the Matrix Variational Masked Auto-Encoder (M-MAE), a novel method that leverages the matrix-based estimation of entropy as a regularizer and subsumes U-MAE as a special case. 
The empirical evaluations underscore the effectiveness of M-MAE compared with the state-of-the-art methods, including a $3.9\%$ improvement in linear probing ViT-Base, and a $1\%$ improvement in fine-tuning ViT-Large, both on ImageNet. 
\end{abstract}
\section{Introduction}

Self-supervised learning (SSL) has demonstrated remarkable advancements across various tasks, including image classification and segmentation, often surpassing the performance of supervised learning approaches ~\citep{chen2020simple, caron2021emerging, li2021self, zbontar2021barlow, bardes2021vicreg}. Broadly, SSL methods can be categorized into three types: contrastive learning, feature decorrelation-based learning, and masked image modeling.

One prominent approach in contrastive self-supervised learning is SimCLR~\citep{chen2020simple}, which employs the InfoNCE loss~\citep{oord2018representation} to facilitate the learning process. Interestingly, \citet{oord2018representation} show that InfoNCE loss can serve as a surrogate loss for the mutual information between two augmented views. 
Unlike contrastive learning which needs to include large amounts of negative samples to ``contrast'', another line of work usually operates without explicitly contrasting with negative samples which are usually called feature decorrelation-based learning \citep{garrido2022duality}, e.g.,  BYOL~\citep{grill2020bootstrap}, SimSiam~\citep{chen2021exploring}, Barlow Twins~\citep{zbontar2021barlow}, VICReg~\citep{bardes2021vicreg}, etc. These methods have garnered attention from researchers seeking to explore alternative avenues for SSL beyond contrastive approaches.

On a different path, the masked autoencoder (MAE)~\citep{he2022masked} introduces a different way to tackle self-supervised learning. Unlike contrastive and feature decorrelation-based methods that learn useful representations by exploiting the invariance between augmented views, MAE employs a masking strategy to have the model deduce the masked patches from visible patches. Therefore, the representation of MAE carries valuable information for downstream tasks.

At first glance, 
these three types of self-supervised learning methods may seem distinct, but researchers have made progress in understanding their connections. \citet{garrido2022duality} establish a duality between contrastive and feature decorrelation-based methods, shedding light on their fundamental connections and complementarity. Additionally, \citet{balestriero2022contrastive} unveil the links between popular feature decorrelation-based SSL methods and dimension reduction methods commonly employed in traditional unsupervised learning. These findings contribute to our understanding of the theoretical underpinnings and potential applications of feature decorrelation-based SSL techniques. However, compared to connections between contrastive and feature decorrelation-based methods, the relationship between MAE and contrastive or feature decorrelation-based methods remains largely unknown. To the best of our knowledge, \cite{zhang2022mask} is the only paper that relates MAE to the alignment term in contrastive learning. 

Though progress has been made in understanding the existing self-supervised learning methods, the tools used in the literature are diverse. As contrastive and feature decorrelation-based learning usually use two augmented views of the same image, one prominent approach is analyzing the mutual information between two views~\citep{oord2018representation, shwartz2023information, shwartz2023compress}. A unified toolbox to understand and improve self-supervised methods is needed. Recently, \cite{bach2022information, skean2023dime} have considered generalizing the traditional information-theoretic quantities to the matrix regime. Interestingly, we find these quantities can be powerful tools in understanding and improving existing self-supervised methods regardless of whether they are contrastive, feature decorrelation-based, or masking-based \citep{he2022masked}. 

Taking the matrix information theoretic perspective, we analyze some prominent contrastive and feature decorrelation-based losses and prove that both Barlow Twins and spectral contrastive learning \citep{haochen2021provable} are maximizing mutual information and joint entropy, see Theorem~\ref{MI max 1} and Theorem~\ref{thm:joint-entropy 1}. These claims are crucial for analyzing contrastive and feature decorrelation-based methods, offering a cohesive and elegant understanding. More interestingly, the same analytical framework extends to MAE as well, wherein the concepts of mutual information and joint entropy gracefully degenerate to entropy. Propelled by this observation, we augment the MAE loss with a matrix-based estimation of entropy, giving rise to our new method, Matrix variational Masked Auto-Encoder (M-MAE), which subsumes U-MAE as a special case, see Theorem~\ref{thm:subsume}.

Empirically, M-MAE stands out with commendable performance. Specifically, it has achieved a $3.9\%$ improvement in linear probing ViT-Base, and a $1\%$ improvement in fine-tuning ViT-Large, both on ImageNet. This empirical result not only underscores the efficacy of M-MAE but also accentuates the potential of matrix information theory in ushering advancements in self-supervised learning paradigms.

In summary, our contributions can be listed as follows:
\begin{itemize}
    \item We use matrix information-theoretic tools like matrix mutual information and joint entropies to understand existing contrastive and feature decorrelation-based self-supervised methods.
    \item We introduce a novel method, M-MAE, which is rooted in matrix information theory, and subsumes U-MAE as a special case. 
    \item Our proposed M-MAE has demonstrated remarkable empirical performance, showcasing a notable improvement in self-supervised learning benchmarks. 
\end{itemize}

\section{Related Work}

\paragraph{Self-supervised learning.}

Contrastive and feature decorrelation based methods have emerged as powerful approaches for unsupervised representation learning. 
By leveraging diverse views or augmentations of input data, they aim to capture meaningful and informative representations that can generalize across different tasks and domains ~\citep{chen2020simple, hjelm2018learning, wu2018unsupervised, tian2019contrastive, chen2021exploring, gao2021simcse, bachman2019learning, 
oord2018representation, ye2019unsupervised, henaff2020data, misra2020self, caron2020unsupervised, haochen2021provable, caron2021emerging,li2021self, zbontar2021barlow, tsai2021note, bardes2021vicreg, tian2020makes, robinson2021contrastive, dubois2022improving}. 


Inspired by the widely adopted Masked Language Modeling (MLM) paradigm in NLP, such as BERT \citep{devlin2018bert},
Masked Image Modeling (MIM) \citep{zhang2022survey} has gained attention as a visual representation learning approach. Notably, several MIM methods, including iBOT \cite{zhou2021ibot}, SimMIM \citep{xie2022simmim}, and MAE \citep{he2022masked}, have demonstrated promising results in this domain. 

\paragraph{Matrix information theory.}
Recently, 
there have been attempts to generalize information theory to measure the relationships between matrices \citep{bach2022information, skean2023dime, zhang2023kernel,zhang2023relationmatch}. The idea is to apply the traditional information-theoretic quantities on the spectrum of matrices. \citep{zhang2023kernel} discuss the relationship between matrix entropy and effective rank. They also discuss the relationship between matrix KL divergence, total coding rate, and matrix entropy, and propose loss to improve the feature decorrelation-based method. \citep{liu2022self} use total coding rate to understand the feature decorrelation-based methods.

\paragraph{Theoretical understanding of self-supervised learning.}

The practical achievements of contrastive learning have ignited a surge of theoretical investigations into the understanding how contrastive loss works \cite{arora2019theoretical, haochen2021provable, haochen2022beyond, tosh2020contrastive, tosh2021contrastive, lee2020predicting, wang2022chaos, nozawa2021understanding, huang2021towards, tian2022deep, hu2022your, tan2023contrastive}. \cite{wang2020understanding} provide an insightful analysis of the optimal solutions of the InfoNCE loss, providing insights into the alignment term and uniformity term that constitute the loss, thus contributing to a deeper understanding of self-supervised learning. \cite{haochen2021provable, wang2022chaos, tan2023contrastive} explore contrastive self-supervised learning methods from a spectral graph perspective. 
Several theoretical investigations have delved into the realm of feature decorrelation based methods within the domain of self-supervised learning, as evidenced by a collection of notable studies~\citep{wen2022mechanism, tian2021understanding, garrido2022duality, balestriero2022contrastive, tsai2021note, pokle2022contrasting, tao2022exploring, lee2021predicting}. 


Compared to contrastive and feature decorrelation based methods, the theoretical understanding of masked image modeling is still in an early stage. \citet{cao2022understand} use the viewpoint of the integral kernel to understand MAE. \citet{zhang2022mask} use the idea of a masked graph to relate MAE with the alignment loss in contrastive learning. Recently, \citet{kong2023understanding} show MAE effectively detects and identifies a specific group of latent variables using a hierarchical model.

\section{Background}
 
\subsection{Matrix information-theoretic quantities} \label{mat info quantity}

In this subsection, we assume \textbf{all the mentioned matrices are positive semi-definite} and follow the constraint that all their \textbf{diagonal elements are $1$}.

We shall first provide the definition of (matrix) entropy as follows:

\begin{definition}[Matrix-based $\alpha$-order (R\'enyi) entropy~\citep{skean2023dime}] Suppose matrix $\mathbf{K}_1 \in \mathbb{R}^{n \times n}$ and $\alpha$ is a positive real number. The $\alpha$-order (R\'enyi) entropy for matrix $\mathbf{K}_1$ is defined as follows:
$$
\operatorname{H}_\alpha\left(\mathbf{K}_1\right)=\frac{1}{1-\alpha} \log \left[\operatorname{tr}\left(\left(\frac{1}{n} \mathbf{K}_1 \right)^\alpha\right)\right],
$$
where $\mathbf{K}^{\alpha}_1$ is the matrix power.

The case of $\alpha=1$ is defined as the von Neumann (matrix) entropy \citep{von2013mathematische}, i.e. 
$$
\operatorname{H}_1\left(\mathbf{K}_1\right)=-\operatorname{tr}\left(\frac{1}{n} \mathbf{K}_1 \log \frac{1}{n} \mathbf{K}_1 \right),
$$
\end{definition}
where $\log$ is the matrix logarithm.

Using the definition of matrix entropy, we can define matrix mutual information and joint entropy as follows.

\begin{definition}[Matrix-based mutual information~\citep{skean2023dime}] \label{matrix mutual information} Suppose matrix $\mathbf{K}_1, \mathbf{K}_2 \in \mathbb{R}^{n \times n}$ and $\alpha$ is a positive real number. The $\alpha$-order (R\'enyi) mutual information for matrices $\mathbf{K}_1$ and $\mathbf{K}_2$ is defined as follows:
$$
\operatorname{I}_{\alpha}(\mathbf{K}_1; \mathbf{K}_2) = \operatorname{H}_{\alpha}(\mathbf{K}_1) + \operatorname{H}_{\alpha}(\mathbf{K}_2) - \operatorname{H}_{\alpha}(\mathbf{K}_1 \odot \mathbf{K}_2).
$$
\end{definition}

\begin{definition}[Matrix-based joint entropy~\citep{skean2023dime}] Suppose matrix $\mathbf{K}_1, \mathbf{K}_2 \in \mathbb{R}^{n \times n}$ and $\alpha$ is a positive real number. The $\alpha$-order (R\'enyi) joint-entropy for matrices $\mathbf{K}_1$ and $\mathbf{K}_2$ is defined as follows:
$$
\operatorname{H}_{\alpha}(\mathbf{K}_1, \mathbf{K}_2) = \mathbf{H}_{\alpha}(\mathbf{K}_1 \odot \mathbf{K}_2),
$$ where $\odot$ is the (matrix) Hadamard product.
\end{definition}

\subsection{Canonical self-supervised learning losses}


We shall recap some canonical losses used in self-supervised learning. As we roughly characterize self-supervised learning into contrastive learning, feature decorrelation-based learning, and masked image modeling. We shall introduce the canonical losses used in these areas sequentially.

In contrastive and feature decorrelation-based learning, people usually adopt the Siamese architecture (dual networks), namely using two parameterized networks: the online network ${f}_{\theta}$ and the target network ${f}_{\phi}$. To create different perspectives of a batch of $B$ data points $\{ \mathbf{x}_i \}^B_{i=1}$, we randomly select an augmentation $\mathcal{T}$ from a predefined set $\tau$ and use it to transform each data point, resulting in new representations $\mathbf{z}^{(1)}_i={f}_{\theta}(\mathcal{T}(\mathbf{x}_i)) \in \mathbb{R}^d$ and $\mathbf{z}^{(2)}_i = {f}_{\phi}( \mathbf{x}_i) \in \mathbb{R}^d$ generated by the online and target networks, respectively. We then combine these representations into matrices $\mathbf{Z}_1=[ \mathbf{z}^{(1)}_1, \ldots ,\mathbf{z}^{(1)}_B]$ and $\mathbf{Z}_2=[\mathbf{z}^{(2)}_1 ,\ldots ,\mathbf{z}^{(2)}_B]$, we assume $\|\mathbf{z}^{(k)}_j \|_2=1$ ($k=1.2$ and $j=1,\cdots,B$). Denote the (batch normalized) vectors for each dimension $i$ ($1 \leq i \leq d$) of the online and target networks as $\bar{\mathbf{z}}^{(1)}_i$ and $\bar{\mathbf{z}}^{(2)}_i$, i.e. coordinate-wise $$ \bar{\mathbf{z}}^{(k)}_i(j)= \frac{\mathbf{z}^{(k)}_{j}(i)}{\sqrt{ \sum^B_{ j=1} (\mathbf{z}^{(k)}_{j}(i))^2 } }$$ ($k=1.2$ and $i=1,\cdots,d$, and $j=1,\cdots,B$). We also define $\bar{\mathbf{Z}}_1 = [\bar {\mathbf{z}}^{(1)}_1 \cdots \bar{\mathbf{z}}^{(1)}_d]^{\top}$ and $\bar{\mathbf{Z}}_2 = [\bar {\mathbf{z}}^{(2)}_1 \cdots \bar{\mathbf{z}}^{(2)}_d]^{\top}$, where $\bar {\mathbf{z}}^{(k)}_i = [\bar {\mathbf{z}}^{(k)}_i(1) \cdots \bar {\mathbf{z}}^{(k)}_i(B)]^{\top}$ ($k=1,2$). 

The idea of contrastive learning is to make the representation of similar objects align and dissimilar objects apart. One of the widely adopted losses in contrastive learning is InfoNCE loss \citep{chen2020simple}, which is defined as follows:
\begin{align}
\mathcal{L}_{\text{InfoNCE}} =& - \frac{1}{2} ( \sum^{B}_{i=1} \log \frac{\exp{((\mathbf{z}^{(1)}_i)^{\top} \mathbf{z}^{(2)}_i)}}{ \sum^B_{j=1} \exp{((\mathbf{z}^{(1)}_i)^{\top} \mathbf{z}^{(2)}_j }) } \nonumber \\  &+   \sum^{B}_{i=1} \log \frac{\exp{((\mathbf{z}^{(2)}_i)^{\top} \mathbf{z}^{(1)}_i)}}{ \sum^B_{j=1} \exp{((\mathbf{z}^{(2)}_i)^{\top} \mathbf{z}^{(1)}_j }) }  ).    
\end{align}

As the InfoNCE loss may be difficult to analyze theoretically, \citet{haochen2021provable} then propose spectral contrastive loss as a good surrogate for InfoNCE. The loss is defined as follows:
\begin{equation} \label{spectral contrastive loss}
\mathcal{L}_{SC} = \sum^B_{i=1} \mid \mid \mathbf{z}^{(1)}_i - \mathbf{z}^{(2)}_i \mid \mid^{2}_2 + \lambda \sum_{ i \neq j} ((\mathbf{z}^{(1)}_i)^{\top} \mathbf{z}^{(2)}_j)^2,    
\end{equation}
where $\lambda$ is a hyperparameter. (Here, we slightly generalize the initial loss a bit, the initial paper sets $\lambda=1$.)

The idea of feature decorrelation-based learning is to learn useful representation by decorrelating features and not explicitly distinguish negative samples. Some notable losses involve VICReg \citep{bardes2021vicreg}, and Barlow Twins \citep{zbontar2021barlow}.
The Barlow Twins loss is given as follows:
\begin{equation}  \label{barlow twins loss}
\mathcal{L}_{BT} = \sum^d_{i=1}\left(1-\mathcal{C}_{i i}\right)^2+\lambda  \sum^d_{i=1} \sum_{j \neq i} \mathcal{C}_{i j}{ }^2    ,
\end{equation}
where $\lambda$ is a hyperparameter and $\mathcal{C}_{i j} = (\bar{\mathbf{z}}^{(1)}_i)^{\top} \bar{\mathbf{z}}^{(2)}_j$ is the cross-correlation coefficient.


The idea of masked image modeling is to learn useful representations by generating the representation from partially visible patches and predicting the rest of the image from the representation, thus useful information in the image remains in the representation. We shall briefly introduce MAE~\citep{he2022masked} as an example. In masked image modeling, people usually adopt only one branch and do not use Siamese architecture. Given a batch of images $\{ \mathbf{x}_i \}^B_{i=1}$, we shall first partition each of the images into $n$ disjoint patches $\mathbf{x}_i = \mathbf{x}_i(j)$ ($1 \leq j \leq n$). Then $B$ random mask vectors $\mathbf{m}_i \in \{0,1\}^n$ will be generated, and denote the two images generated by these masks as 
\begin{equation}
\mathbf{x}^{(1)}_i = \mathbf{x}_i \odot \mathbf{m}_i \quad \text{and} \quad \mathbf{x}^{(2)}_i = \mathbf{x}_i \odot (1-\mathbf{m}_i).    
\end{equation}

The model consists of two modules: an encoder $f$ and a decoder $g$. The encoder transform each view $\mathbf{x}^{(1)}_i$ into a representation $\mathbf{z}_i=f(\mathbf{x}^{(1)}_i)$. The loss function is $\sum^{B}_{i=1} \| g(\mathbf{z}_i) - \mathbf{x}^{(2)} \|_2^2$. We also denote the representations in a batch as $\mathbf{Z}=[\mathbf{z}_1, \cdots, \mathbf{z}_B]$.

Finally, we will present the U-MAE loss \citep{zhang2022mask} as:
\begin{equation*}
\mathcal{L}_{\text{U-MAE}} = \mathcal{L}_{\operatorname{MAE}} + \gamma  \sum_{i \neq j} (\mathbf{z}^{\top}_i \mathbf{z}_j)^2 ,    
\end{equation*}
where $\gamma$ is a hyper-parameter.

The goal of this paper is to use a matrix information maximization viewpoint to understand the seemingly different losses in contrastive and feature decorrelation-based methods. We would like also to use matrix information-theoretic tools to improve MAE. We only analyze $4$ popular losses: spectral contrastive, Barlow Twins, MAE, and U-MAE. All the proofs can be found in Appendix \ref{sec:proofs}. More experiments can be found in Appendix \ref{more exp}.

\section{Applying matrix information theory to contrastive and feature decorrelation-based methods}
As we have discussed in the preliminary session, in contrastive and feature decorrelation-based methods, a common practice is to use two branches (Siamese architecture) namely an online network and a target network to learn useful representations. However, the relationship of the two branches during the training process is mysterious. In this section, we shall use matrix information quantities to unveil the complicated relationship in Siamese architectures.

\subsection{Measuring the mutual information}

One interesting derivation in \cite{oord2018representation} is that it can be shown that
\begin{equation} \label{infonce bound}
\mathcal{L}_{\text{InfoNCE}} \geq -\operatorname{I}(\mathbf{Z}^{(1)} ; \mathbf{Z}^{(2)}) + \log B,    
\end{equation}
where $\mathbf{Z}^{(i)}$ denotes the sampled distribution of the representation.

Though InfoNCE loss is a promising surrogate for estimating the mutual information between the two branches in self-supervised learning. \cite{sordoni2021decomposed} doubt its effectiveness when facing high-dimensional inputs, where the mutual information may be larger than $\log B$, making the bound vacuous. Then a natural question arises: Can we calculate the mutual information exactly? Unfortunately, it is hard to calculate the mutual information reliably and effectively. Thus we want to see the effect of changing our strategy by using the \emph{matrix} mutual information instead of the traditional one.

As the matrix mutual information has a closed-form expression with only requires a few mild conditions on the input matrices (please refer to section ~\ref{mat info quantity}), one question remains: How to choose the matrices used in the (matrix) mutual information? We find the (batch normalized) sample covariance matrix and batch gram matrix of $l_2$ normalized representations serve as good candidates. The reason is that by using normalization, the covariance and gram matrices naturally satisfy the requirements that: All the diagonals equal to $1$, the matrix is positive semi-definite and it is easy to estimate from data samples.

\begin{figure}[t] 
\centering 
\includegraphics[width=0.9\columnwidth]{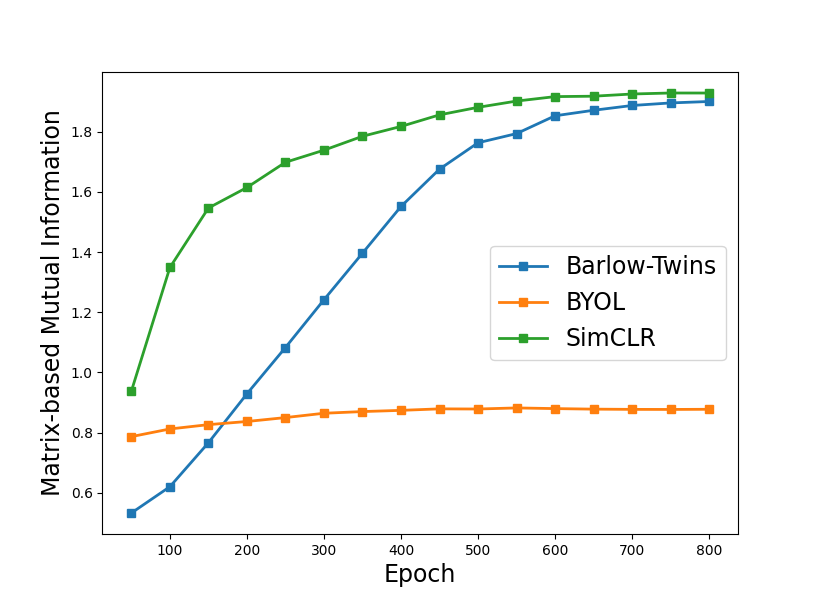}
    \caption{Visualization of matrix-based mutual information on CIFAR10 for Barlow-Twins, BYOL, and SimCLR.}
    \label{fig:MMI}
\end{figure}

Notably, the covariance and garm matrix can be seen to have an informal ``duality'' \citep{garrido2022duality}. Specifically, the sample covariance matrix can be expressed as $\mathbf{Z}\mathbf{Z}^{\top} \in \mathbb{R}^{d \times d}$ and the batch-sample Gram matrix can be expressed as $\mathbf{Z}^{\top}\mathbf{Z} \in \mathbb{R}^{B \times B}$. The closeness of these two matrices makes us call $B$ and $d$ has duality. As matrix information theory can not only deal with samples from batches but also can exploit the relationship among batches. This makes this theory well-suited for analyzing self-supervised learning methods. 

Notably, spectral contrastive loss (Eqn.~\ref{spectral contrastive loss}) is a good surrogate loss for InfoNCE loss and calculates the loss involving the batch gram matrix. Another famous loss used in feature decorrelation-based methods is the Barlow Twins (Eqn.~\ref{barlow twins loss}), which involves the batch-normalized sample covariance matrix. Therefore, these two losses will be our main focus for theoretical investigation.

As traditional information theory provides a bound Eqn.~(\ref{infonce bound}), thus we are interested in investigating whether there is a matrix information type bound. In the following, we will show that spectral contrastive loss and Barlow Twins loss have (matrix) mutual information bound. Specifically, for ease of theoretical analysis, we first consider setting the $\alpha$ in entropy to be $2$.

To prove the bound, we will first present a proposition that relates the mutual information with the Frobenius norm.

\begin{proposition} \label{mutual information reduction}
$\operatorname{I}_2(\mathbf{K}_1; \mathbf{K}_2) = 2\log d - \log \frac{|| \mathbf{K}_1 ||^2_F || \mathbf{K}_2 ||^2_F}{|| \mathbf{K}_1 \odot \mathbf{K}_2 ||^2_F}$, where $d$ is the size of matrix $\mathbf{K}_1$.
\end{proposition}

We will also need a technical proposition that relates the cross-correlation and auto-correlation.
\begin{lemma} \label{cross and cov}
Suppose $\mathbf{a}$, $\mathbf{b}$, $\mathbf{a}'$ and $\mathbf{b}'$ are $l_2$ normalized, then $|\mathbf{a}^{\top} \mathbf{b}  |  \leq |\mathbf{a}^{\top} \mathbf{b}'  | + \| \mathbf{b} - \mathbf{b}'\| = |\mathbf{a}^{\top} \mathbf{b}'  | +  \sqrt{2(1 - \mathbf{b}^{\top} \mathbf{b}')}$.
\end{lemma}

Using the previous two propositions, we can derive the following bounds that relate the matrix mutual information and the loss value. 

\begin{theorem}  \label{MI bound}
1. For the spectral contrastive loss, we have 
\begin{equation*}
\operatorname{I}_2(\mathbf{Z}^{\top}_1\mathbf{Z}_1 ; \mathbf{Z}^{\top}_2\mathbf{Z}_2) \geq \log B - 2 \log(1+ (2 + \frac{2}{B \lambda}) \mathcal{L}_{SC} ).
\end{equation*}
2. For the Barlow Twins' loss, we have 
\begin{equation*}
\operatorname{I}_2(\bar{\mathbf{Z}}_1\bar{\mathbf{Z}}^{\top}_1 ; \bar{\mathbf{Z}}_2\bar{\mathbf{Z}}^{\top}_2) \geq \log d - 2 \log(1+ \frac{2}{d \lambda} \mathcal{L}_{BT} + 4 \sqrt{d\mathcal{L}_{BT}} ).
\end{equation*}
\end{theorem}

\begin{proof}
We will only present the proof for spectral contrastive loss as Barlow Twins loss is similar.

By Proposition \ref{mutual information reduction}, we know that $\operatorname{I}_2(\mathbf{Z}^{\top}_1\mathbf{Z}_1; \mathbf{Z}^{\top}_2\mathbf{Z}_2) = 2\log B - \log \frac{|| \mathbf{Z}^{\top}_1\mathbf{Z}_1 ||^2_F || \mathbf{Z}^{\top}_2\mathbf{Z}_2 ||^2_F}{|| \mathbf{Z}^{\top}_1\mathbf{Z}_1 \odot \mathbf{Z}^{\top}_2\mathbf{Z}_2 ||^2_F} \geq 2\log B - \log \frac{|| \mathbf{Z}^{\top}_1\mathbf{Z}_1 ||^2_F || \mathbf{Z}^{\top}_2\mathbf{Z}_2 ||^2_F}{B} = \log B - \log { \frac{|| \mathbf{Z}^{\top}_1\mathbf{Z}_1 ||^2_F}{B} \frac{|| \mathbf{Z}^{\top}_2\mathbf{Z}_2 ||^2_F}{B}}.$

On the other hand, $\frac{|| \mathbf{Z}^{\top}_1\mathbf{Z}_1 ||^2_F}{B} = 1 + \frac{\sum_{i \neq j} ((\mathbf{z}^{(1)}_i)^{\top}  \mathbf{z}^{(1)}_j)^2 }{B}.$

Using Lemma \ref{cross and cov}, we know that $((\mathbf{z}^{(1)}_i)^{\top}  \mathbf{z}^{(1)}_j)^2 \leq (|(\mathbf{z}^{(1)}_i)^{\top}  \mathbf{z}^{(2)}_j| +  \| \mathbf{z}^{(1)}_j - \mathbf{z}^{(2)}_j \| )^2 \leq 2(|(\mathbf{z}^{(1)}_i)^{\top}  \mathbf{z}^{(2)}_j|^2 +  \| \mathbf{z}^{(1)}_j - \mathbf{z}^{(2)}_j \|^2).$

Therefore, $\sum_{i \neq j} ((\mathbf{z}^{(1)}_i)^{\top}  \mathbf{z}^{(1)}_j)^2 \leq 2 (\sum_{i \neq j} |(\mathbf{z}^{(1)}_i)^{\top}  \mathbf{z}^{(2)}_j|^2 + \sum_{i \neq j}  \| \mathbf{z}^{(1)}_j - \mathbf{z}^{(2)}_j \|^2 ) < 2(\sum_{i \neq j} |(\mathbf{z}^{(1)}_i)^{\top}  \mathbf{z}^{(2)}_j|^2 + B \sum_j  \| \mathbf{z}^{(1)}_j - \mathbf{z}^{(2)}_j \|^2) \leq 2(\frac{1}{\lambda} + B) \mathcal{L}_{SC}$.

Combining all the above, the conclusion follows.
\end{proof}

What about the mutual information when $\alpha \neq 2$. For example $\alpha=1$? We then plot the mutual information of covariance matrices between branches in Figure~\ref{fig:MMI}. We can find out that the mutual information increases during training, which is similar to the case of $\alpha=2$ proved by Theorem \ref{MI bound}. More interestingly, the mutual information of SimCLR and Barlow Twins meet at the end of the training, strongly emphasizing the duality of these algorithms. The empirical findings motivate us to consider the case of general $\alpha>0$. 

Unfortunately, it is hard for us to provide bounds similar to Theorem \ref{MI bound} for general $\alpha>0$. But interestingly, we find the following interesting theorem. 

\begin{theorem} \label{MI max 1}
When $\alpha>0$, {Barlow Twins and spectral contrastive learning losses maximize the matrix mutual information when their loss value is $0$.} 
\end{theorem}

The proof of theorem \ref{MI max 1}  relies on the following upper bound (Proposition \ref{mutual information upper bound}). The key idea is when the loss value is $0$, the mutual information can be explicitly calculated and meets the upper bound.

\begin{proposition} \label{mutual information upper bound}
Suppose $\mathbf{K}_1$ and $\mathbf{K}_2$ are $d \times d$ positive semi-definite matrices with the constraint that each of its diagonals is $1$. Then $\operatorname{I}_{\alpha}(\mathbf{K}_1; \mathbf{K}_2) \leq \log d$.
\end{proposition}

By combining Theorems \ref{MI bound} and \ref{MI max 1}, we can conclude the following corollary.

\begin{corollary} \label{MI max}
When $\alpha=2$, {the bounds given by Theorem \ref{MI bound} is tight when loss values are $0$.} 
\end{corollary}

From the above theorems, we know that when minimizing the spectral contrastive loss and Barlow Twins loss, the mutual information follows a trajectory towards its maximum. This can be seen as mitigating the slight drawback of bound Eqn.~(\ref{infonce bound}) in that it only provides an inequality and does not discuss the optimal point.

\subsection{Measuring the (joint) entropy}

After discussing the application of matrix mutual information in self-supervised learning. We wonder how another import quantity (joint entropy) evolves during the process.

We can show that the matrix joint entropy can indeed reflect the dimensions of representations in Siamese architectures through the following Proposition \ref{joint entropy upper bound}.

\begin{proposition} \label{joint entropy upper bound}
The joint entropy lower bounds the representation rank in two branches by having the inequality as follows:
\begin{align*}
&\operatorname{H}_1(\mathbf{K}_1, \mathbf{K}_2) \leq \log(\operatorname{rank}(\mathbf{K}_1 \odot \mathbf{K}_2)) \\
&\leq \log \operatorname{rank}(\mathbf{K}_1) + \log \operatorname{rank}(\mathbf{K}_2).  
\end{align*}
\begin{align*}
&\max \{\operatorname{H}_{\alpha}(\mathbf{K}_1), \operatorname{H}_{\alpha}(\mathbf{K}_2) \} \\
&\leq \operatorname{H}_{\alpha}(\mathbf{K}_1, \mathbf{K}_2)  \leq \operatorname{H}_{\alpha}(\mathbf{K}_1) + \operatorname{H}_{\alpha}(\mathbf{K}_2).  
\end{align*}
\end{proposition}

This proposition shows that the bigger the joint entropy between the two branches is, the less likely that the representation (rank) collapse. Interestingly, similar results can be proven for (traditional) entropy surrogates \citep{yu2020learning}, which we will briefly introduce as follows.

We shall introduce a matrix-based surrogate for entropy as follows:
\begin{definition} \label{TCR loss}
Suppose $B$ samples $\mathbf{Z} = [\mathbf{z}_1, \mathbf{z}_2, \cdots, \mathbf{z}_B] \in \mathbb{R}^{d \times B}$ are i.i.d. samples from a distribution $p(z)$. Then the total coding rate (TCR) \citep{yu2020learning} of $p(z)$ is defined as follows:
\begin{equation}
\text{TCR}_{\mu}(\mathbf{Z}) = \log \operatorname{det}(\mu \mathbf{I}_d + \mathbf{Z}\mathbf{Z}^{\top}),
\end{equation}
where $\mu$ is a non-negative hyperparameter.
\end{definition}
For notation simplicity, we shall also write $\text{TCR}_{\mu}(\mathbf{Z})$ as $\text{TCR}_{\mu}(\mathbf{Z}\mathbf{Z}^{\top})$. Notably, there is a close relationship between TCR and matrix entropy, which is presented in the following theorem through the lens of matrix KL divergence \ref{matrix kl}. The key is utilizing the asymmetries of the matrix KL divergence \citep{zhang2023kernel}.

\begin{definition}[Matrix KL divergence \citep{bach2022information}] \label{matrix kl}
Suppose matrices $\mathbf{K}_1, \mathbf{K}_2 \in \mathbb{R}^{n \times n}$ which $\mathbf{K}_1(i,i) = \mathbf{K}_2(i,i) = 1$ for every $i=1, \cdots, n$. Then the Kullback-Leibler (KL) divergence between two positive semi-definite matrices $\mathbf{K}_1$ and $\mathbf{K}_2$ is defined as
\begin{equation*}
\operatorname{KL}\left(\mathbf{K}_1 \mid \mid \mathbf{K}_2 \right)=\operatorname{tr}\left[\mathbf{K}_1\left(\log \mathbf{K}_1 - \log \mathbf{K}_2 \right)\right].
\end{equation*}
\end{definition}

\begin{proposition} \label{TCR entropy relation}
Suppose $\mathbf{K}$ is a $d \times d$ matrix with the constraint that each of its diagonals is $1$. Then the following equalities holds:
\begin{equation}
\begin{aligned}
&\operatorname{H}_1(\mathbf{K}) = \log d - \frac{1}{d} \operatorname{KL}(\mathbf{K}, \mathbf{I}_d),\\
&
\operatorname{TCR}_{\mu}(\mathbf{K}) = d \log(1+\mu) -\operatorname{KL}(\mathbf{I}_d, \frac{1}{1+\mu} (\mu \mathbf{I}_d + \mathbf{K})).
\end{aligned}
\end{equation}
\end{proposition}

\begin{figure}[t] 
\centering 
\includegraphics[width=0.9\columnwidth]{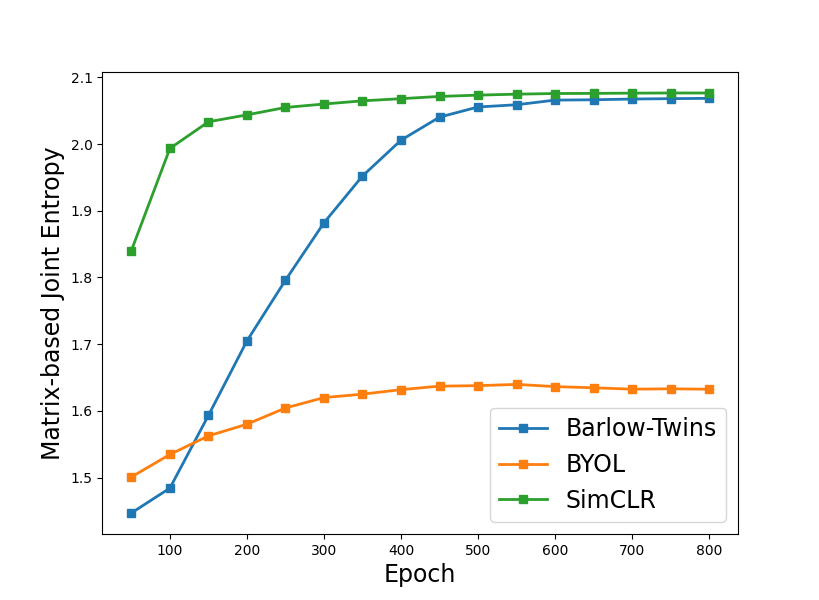}
    \caption{Visualization of matrix-based joint entropy on CIFAR10 for Barlow-Twins, BYOL and SimCLR.}
    \label{fig:MJE}
\end{figure}

As TCR can be treated as a good surrogate for entropy, we can obtain the following bound for its joint entropy version.

\begin{proposition} \label{joint coding rate}
The (joint) total coding rate upper bounds the rate in two branches by having the inequality as follows:
\begin{equation}
\operatorname{TCR}_{\mu^2+2 \mu}(\mathbf{K}_1 \odot \mathbf{K}_2) \geq \operatorname{TCR}_{\mu}(\mathbf{K}_1) + \operatorname{TCR}_{\mu}(\mathbf{K}_2).   
\end{equation}
\end{proposition}
Combining Propositions~\ref{TCR entropy relation}, \ref{joint entropy upper bound}, and~\ref{joint coding rate}, it is clear that the bigger the entropy is for each branch the bigger the joint entropy. Thus by combining the conclusion from the above two theorems, it is evident that the joint (matrix or TCR) entropy strongly reflects the extent of collapse during training.

We will then show a bound relating to the matrix joint entropy and the loss values. This remarkable conclusion is proved when the Renyi entropy order $\alpha=2$.

We shall first present a proposition that relates the joint entropy with the Frobenius norm. 

\begin{proposition} \label{joint entropy reduction}
Suppose $\mathbf{K}_1, \mathbf{K}_2\in \mathbb{R}^{d \times d}$. Then
$\operatorname{H}_2(\mathbf{K}_1, \mathbf{K}_2) = 2 \log d - \log \mid \mid \mathbf{K}_1 \odot \mathbf{K}_2 \mid \mid^2_F$, where $F$ is the Frobenius norm.
\end{proposition}

We can use Lemma \ref{cross and cov} and Proposition \ref{joint entropy reduction} to bind the matrix joint entropy with the loss values.

\begin{theorem}  \label{joint entropy loss bound}
1. In the spectral contrastive loss, we have 
\begin{equation*}
\operatorname{H}_2(\mathbf{Z}^{\top}_1\mathbf{Z}_1, \mathbf{Z}^{\top}_2\mathbf{Z}_2) \geq \log B -  \log(1+ (2 + \frac{2}{B \lambda}) \mathcal{L}_{SC} ).
\end{equation*}

2. In the Barlow Twins loss, we have 
\begin{equation*}
\operatorname{H}_2(\bar{\mathbf{Z}}_1\bar{\mathbf{Z}}^{\top}_1, \bar{\mathbf{Z}}_2\bar{\mathbf{Z}}^{\top}_2) \geq \log d -  \log(1+ \frac{2}{d \lambda} \mathcal{L}_{BT} + 4 \sqrt{d\mathcal{L}_{BT}} ).
\end{equation*}
\end{theorem}

What about the joint entropy when $\alpha=1$ behaves empirically? We shall then plot the joint entropy of covariance matrices between branches in Figure \ref{fig:MJE}. We can find out that the joint entropy increases during training. More interestingly, the joint entropy of SimCLR and Barlow Twins meet at the end of training, strongly reflects a duality of these algorithms. 

Motivated by the above bound, one may wonder what will happen to the joint entropy when the loss is optimized to $0$ for general $\alpha>0$.

Then we can show the following theorem for general $\alpha>0$.
\begin{theorem} \label{thm:joint-entropy 1}
When $\alpha>0$, {Barlow Twins and spectral contrastive learning losses maximize the matrix joint entropy when their loss value is $0$.} 
\end{theorem}

The key to proving theorem \ref{thm:joint-entropy 1} lies in the following proposition that finds the optimal point of entropy. 

\begin{proposition} \label{entropy optimal point}
\begin{equation}
\mathbf{I}_d = \operatorname{argmax} \operatorname{H}_{\alpha}(\mathbf{K}), \operatorname{and} 
\mathbf{I}_d = \operatorname{argmax} 
\operatorname{TCR}_{\mu}(\mathbf{K}),
\end{equation} where the maximization is over $d \times d$ positive semi-definite matrices with the constraint that each of its diagonals is $1$.
\end{proposition}

By combining Theorems \ref{joint entropy loss bound} and \ref{thm:joint-entropy 1}, we can conclude the following corollary.

\begin{corollary} \label{Joint entro max}
When $\alpha=2$, {the bounds given by Theorem \ref{joint entropy loss bound} is tight when loss values are $0$.} 
\end{corollary}

Similarly, one may also prove that
\begin{theorem}\label{TCR bound}
$\operatorname{TCR}_{\mu}(\bar{\mathbf{Z}}_1\bar{\mathbf{Z}}^{\top}_1\odot \bar{\mathbf{Z}}_2\bar{\mathbf{Z}}^{\top}_2) \geq d \log(\mu +1) - \frac{1}{2 \mu^2} ( \frac{2}{\lambda} \mathcal{L}_{BT} + 4(d-1) \sqrt{d\mathcal{L}_{BT}} ).$   
\end{theorem}
\textbf{Remark}: The bound given by theorem \ref{TCR bound} is also tight when the loss value is 0. A similar bound can also be proven batch-wise.

From the above theorems, we know that when minimizing the spectral contrastive loss and Barlow Twins loss, the matrix joint entropy follows a trajectory towards its maximum.

Our theoretical results may show that various contrastive and feature-decorrelation-based methods have similar implicit information maximization processes, thus explaining why they get comparable performance after sufficient training.

\section{Applying matrix information theory to masked image modeling}

As we have previously discussed the central role of mutual information and joint entropy in the contrastive and feature decorrelation-based methods (Which are all Siamese architecture-based due to the use of augmented views of images). We wonder if can we apply this matrix information theory to improve self-supervised methods using only one network, not the Siamese architecture. To the best of our knowledge, the famous vision self-supervised method using only one network is the masked autoencoder type (MAE) \citep{he2022masked}.

From a (traditional) information-theoretic point of view, when the two branches merge into one branch the mutual information $\operatorname{I}(\mathbf{X};\mathbf{X})$ and the joint entropy $\operatorname{H}(\mathbf{X},\mathbf{X})$ both equal to the Shannon entropy $\operatorname{H}(\mathbf{X})$. By Propositions~\ref{TCR entropy relation}, \ref{joint entropy upper bound}, and~\ref{joint coding rate}, one may see that the joint entropy maximization is closely related to each branch's maximization. Additionally, by the conclusion of Theorems \ref{joint entropy loss bound} and \ref{thm:joint-entropy 1}, one may expect a higher entropy during contrastive and feature decorrelation-based methods. Thus we would like to use the matrix-based entropy in MAE training.

Moreover, matrix entropy can be shown to be very close to a quantity called effective rank. And \citet{zhang2022mask, garrido2023rankme} show that the effective rank is a critical quantity for better representation. The definition of effective rank is formally stated in Definition \ref{def:erank} and it is easy to show when the matrix is positive semi-definite and has all its diagonal being $1$ the effective rank is the exponential of the matrix entropy \citep{zhang2023kernel}.

\begin{definition}[Effective Rank~\citep{roy2007effective}]
\label{def:erank}
For a non-all-zero matrix $\mathbf{A} \in \mathbb{R}^{n \times n}$, the effective rank, denoted as $\operatorname{erank}(\mathbf{A})$, is defined as
\begin{equation}
\operatorname{erank}(\mathbf{A}) \triangleq \exp {(\operatorname{H}\left(p_1, p_2, \ldots, p_n\right))},
\end{equation}
where $p_i = \frac{\sigma_i}{\sum_{k=1}^{n} \sigma_k}$, $\{\sigma_i \mid i = 1,\cdots,n \}$ represents the singular values of $\mathbf{A}$, and $\operatorname{H}$ denotes the Shannon entropy.
\end{definition}

Thus it is natural to add the matrix entropy to the MAE loss to give a new self-supervised learning method. As the numerical instability of calculating matrix entropy is larger than its proxy TCR during training, we shall use the TCR loss (definition \ref{TCR loss}), which is a matrix-based estimator for entropy \citep{yu2020learning}.

Recall that we assume each representation $z_i$ is $l_2$ normalized. If we take the latent distribution of ${Z}$ as the uniform distribution on the unit hyper-sphere $S^{d-1}$, we shall get the following loss for self-supervised learning.
\begin{equation}
\mathcal{L}_{\text{M-MAE}} \triangleq \mathcal{L}_{\operatorname{MAE}} - \lambda \cdot \operatorname{TCR}_{\mu}(\mathbf{Z}),   
\end{equation}
where $\lambda$ is a loss-balancing hyperparameter.

The name of matrix variational masked auto-encoder (M-MAE) is due to the reason that we can link this new loss to a traditional unsupervised learning method variational auto-encoder (VAE) \citep{doersch2016tutorial}.

Recall the loss for traditional variational auto-encoder which is given as follows.
\begin{equation*}
\begin{aligned}
\mathcal{L}_{\text{VAE}} \triangleq & \mathbb{E}_{\mathbf{z}}[-\log q(\mathbf{x} | \mathbf{z})+ \operatorname{KL}( p(\mathbf{z} | \mathbf{x}) \| q(\mathbf{z}) )] ,\\
& \text{where } \mathbf{z} \sim p(\mathbf{z} | \mathbf{x}).
\end{aligned}    
\end{equation*}
The loss contains two terms, the first term $-\log q(\mathbf{x} | \mathbf{z})$ is a reconstruction loss that measures the decoding error. The second term is a discriminative term, which measures the divergence of the encoder distribution $p(\mathbf{z} | \mathbf{x})$ with the latent distribution $q(\mathbf{z})$.

We will first show why MAE loss resembles the first term in VAE, i.e. $\mathbb{E}_{\mathbf{z}}[-\log q(\mathbf{x} | \mathbf{z})]$. In the context of masked image modeling, we usually use MSE loss in place of the log-likelihood. For any input image $\mathbf{x}$, the process of randomly generating a masked vector $m$ and obtaining $\mathbf{z}=f(\mathbf{x} \odot \mathbf{m})$ can be seen as modeling the generating process of $\mathbf{z} | \mathbf{x}$. The decoding process $\mathbf{x} | \mathbf{z}$ can be modeled by concatenating $g(\mathbf{z})$ and $\mathbf{x} \odot \mathbf{m}$ by the (random) position induced by $m$. Thus the reconstruction loss will be $\| \text{concat}(g(\mathbf{z}), \mathbf{x} \odot \mathbf{m}) - \mathbf{x} \|^2_2= \| g(\mathbf{z}) - \mathbf{x} \odot (1-\mathbf{m})  \|^2_2$. For a batch of images $\{ \mathbf{x}_i \}^B_{i=1}$, this recovers the MAE loss. 

We will then show how matrix entropy resembles the second term in VAE, i.e. $\mathbb{E}_{\mathbf{z}}[\operatorname{KL}( p(\mathbf{z} | \mathbf{x}) \| q(\mathbf{z}) )]$. This is clear by noticing that $q(\mathbf{z})$ is a latent distribution on the unit hyper-sphere $S^{d-1}$ and we naturally choose it as uniform distribution. By taking the covariance matrix of $p(\mathbf{z} | \mathbf{x})$ and $q(\mathbf{z})$ and using the matrix KL divergence (definition \ref{matrix kl}), this term becomes $\operatorname{KL}(\mathbf{Z}\mathbf{Z}^{\top} \| \mathbf{I}_d)$. By Theorem \ref{TCR entropy relation}, this closely relates to the TCR (and matrix entropy)

Finally, we will present the link of our M-MAE loss to a state-of-the-art one U-MAE \citep{zhang2022mask}.

\begin{theorem}
\label{thm:subsume}
U-MAE is a second-order approximation of our proposed M-MAE.   
\end{theorem}

\begin{proof}
The key point is noticing that representations are $l_2$ normalized and the fact that $\|\mathbf{Z}^{\top}\mathbf{Z}\|^2_F = \operatorname{tr}( (\mathbf{Z}^{\top}\mathbf{Z})^2)$.

Using Taylor expansion, we will have:
\begin{align*}
\mathcal{L}_{\text{M-MAE}} 
& = \mathcal{L}_{\operatorname{MAE}} - \lambda \cdot \log \operatorname{det}(\mathbf{I}_d+ \frac{1}{\mu}\mathbf{Z}\mathbf{Z}^{\top}) + \text{Const.} \\
& = \mathcal{L}_{\operatorname{MAE}} - \lambda \cdot \log \operatorname{det}(\mathbf{I}_B+ \frac{1}{\mu}\mathbf{Z}^{\top}\mathbf{Z}) + \text{Const.} \\
&= \mathcal{L}_{\operatorname{MAE}} - \lambda \cdot  \operatorname{tr} \log(\mathbf{I}_B+ \frac{1}{\mu}\mathbf{Z}^{\top}\mathbf{Z}) + \text{Const.} \\
&= \mathcal{L}_{\operatorname{MAE}} - \lambda \cdot  \operatorname{tr}(\frac{1}{\mu}\mathbf{Z}^{\top}\mathbf{Z} - \frac{1}{2\mu^2} (\mathbf{Z}^{\top}\mathbf{Z})^2 + \cdots) \\
&= \mathcal{L}_{\operatorname{U-MAE}} + \text{Higher-order-terms} + \text{Const}.\qedhere
\end{align*}   
\textbf{Remark}: A proof similar to theorem \ref{TCR bound} will also give a bound that relates M-MAE and U-MAE.

\end{proof}

\section{Experiments}
\label{sec:experiments}

In this section, we empirically evaluate our Matrix Variational Masked Auto-Encoder (M-MAE) with TCR loss, placing special emphasis on its performance in comparison to the U-MAE model with Square uniformity loss as a baseline. This experiment aims to shed light on the benefits that matrix information-theoretic tools can bring to methods based on masked image modeling\footnote{The code is available at \url{https://github.com/yifanzhang-pro/M-MAE}.}.

\subsection{Experimental setup}

\textbf{Datasets: ImageNet-1K.} We utilize the ImageNet-1K dataset~\citep{deng2009imagenet}, which is one of the most comprehensive datasets for image classification. It contains over 1 million images spread across 1000 classes and in self-supervised learning experiments the labels are dropped, providing a robust platform for evaluating our method's generalization capabilities. 

\textbf{Model architectures.} We adopt Vision Transformers (ViT) such as ViT-Base and ViT-Large for our models. We closely follow the precedent settings by the U-MAE~\citep{zhang2022mask} paper, as Theorem \ref{thm:subsume} shows the closeness of this method to our M-MAE loss. 

\textbf{Hyperparameters.} For a fair comparison, we adopt U-MAE's original hyperparameters: a mask ratio of 0.75 and a uniformity term coefficient $\lambda$ of 0.01 by default. Both models are pre-trained for 200 epochs on ImageNet-1K with a batch size of 1024, and weight decay is similarly configured as 0.05 to ensure parity in the experimental conditions. For ViT-Base, we set the TCR coefficients $\mu = 1$, and for ViT-Large, we set $\mu = 3$.

\subsection{Evaluation results}

\textbf{Evaluation metrics.} From Table~\ref{table:fintune-results}, it's evident that the M-MAE loss outperforms both MAE and U-MAE in terms of linear evaluation and fine-tuning accuracy. Specifically, for ViT-Base, M-MAE achieves a linear probing accuracy of 62.4\%, which is a substantial improvement over MAE's 55.4\% and U-MAE's 58.5\%. Similarly, in the context of ViT-Large, M-MAE achieves an accuracy of 66.0\%, again surpassing both MAE and U-MAE. In terms of fine-tuning performance, M-MAE also exhibits superiority, achieving 83.1\% and 84.3\% accuracy for ViT-Base and ViT-Large respectively. Notably, a 1\% increase in accuracy at ViT-Large is very significant. These results empirically validate the theoretical advantages of incorporating matrix information-theoretic tools into masked image modeling, as encapsulated by the TCR loss term in the M-MAE loss function.

\begin{table}[ht]
\centering
\caption{Linear evaluation accuracy (\%) and fine-tuning accuracy (\%) of pretrained models by MAE loss, U-MAE loss, and M-MAE loss with different ViT backbones on ImageNet-1K. The uniformity regularizer TCR loss in the M-MAE loss significantly improves the linear evaluation performance and fine-tuning performance of the MAE loss.}
\vspace{1ex}
\begin{tabular}{llcccc}
\toprule
Downstream Task                                  & Method        & ViT-Base       & ViT-Large      \\ \midrule
\multirow{2}{*}{Linear Probing} & MAE                           & 55.4           &62.2                \\
                                           & U-MAE           & \underline{58.5}           &    \underline{65.8}            \\ 
                                           & M-MAE   & \textbf{62.4}           &    \textbf{66.0}            \\ 
                                           \midrule
\multirow{2}{*}{Fine-tuning}          & MAE          &    82.9            &      \underline{83.3}          \\
                                           & U-MAE           &       \underline{83.0}         &      83.2          \\ 
                                           & M-MAE    & \textbf{83.1}           &    \textbf{84.3}            \\ 
                                           \bottomrule
\end{tabular}
\vspace{-3.7mm}
\label{table:fintune-results}
\end{table}



\section{Conclusion}

In conclusion, this study delves into self-supervised learning (SSL), examining contrastive, feature decorrelation-based learning, and masked image modeling through the lens of matrix information theory. Our exploration reveals that many SSL methods are maximizing matrix information-theoretic quantities like matrix mutual information and matrix joint entropy on Siamese architectures. 

Motivated by the theoretical findings, we also introduce a novel method, the matrix variational masked auto-encoder (M-MAE), enhancing masked image modeling by adding matrix-based estimators for entropy. Empirical results show the effectiveness of the introduced M-MAE loss.

\section*{Acknowledgment}
Yang Yuan is supported by the Ministry of Science and Technology of the People's Republic of China, the 2030 Innovation Megaprojects ``Program on New Generation Artificial Intelligence'' (Grant No. 2021AAA0150000). 

Weiran Huang is supported by the 2023 CCF-Baidu Open Fund and Microsoft Research Asia. 

We would also like to express our sincere gratitude to the reviewers of ICML 2024 for their insightful and constructive feedback. Their valuable comments have greatly contributed to improving the quality of our work.

\section*{Impact Statement}
This paper presents work whose goal is to advance the field of Machine Learning. There are many potential societal consequences of our work, none which we feel must be specifically highlighted here.



\bibliography{reference}
\bibliographystyle{icml2024}


\clearpage
\appendix
\onecolumn
\section{Appendix for proofs}
\label{sec:proofs}

\begin{proposition} 
$\operatorname{I}_2(\mathbf{K}_1; \mathbf{K}_2) = 2\log d - \log \frac{|| \mathbf{K}_1 ||^2_F || \mathbf{K}_2 ||^2_F}{|| \mathbf{K}_1 \odot \mathbf{K}_2 ||^2_F}$, where $d$ is the size of matrix $\mathbf{K}_1$.
\end{proposition}

\begin{proof}
The proof is straightforward by using the definition of matrix mutual information when $\alpha=2$ and the fact that when $\mathbf{K}$ is symmetric $\operatorname{tr}(\mathbf{K}^2)=\operatorname{tr}(\mathbf{K}^{\top} \mathbf{K})=\mid \mid \mathbf{K} \mid \mid^2_F$.    
\end{proof}

\begin{lemma} 
Suppose $\mathbf{a}$, $\mathbf{b}$, $\mathbf{a}'$ and $\mathbf{b}'$ are $l_2$ normalized, then $|\mathbf{a}^{\top} \mathbf{b}  |  \leq |\mathbf{a}^{\top} \mathbf{b}'  | + \| \mathbf{b} - \mathbf{b}'\| = |\mathbf{a}^{\top} \mathbf{b}'  | +  \sqrt{2(1 - \mathbf{b}^{\top} \mathbf{b}')}$.
\end{lemma}

\begin{proof}
Note $|\mathbf{a}^{\top} \mathbf{b}  | = |\mathbf{a}^{\top} \mathbf{b}' + \mathbf{a}^{\top} (\mathbf{b} - \mathbf{b}')  | \leq |\mathbf{a}^{\top} \mathbf{b}'  | + \| \mathbf{a} \| \| \mathbf{b} - \mathbf{b}'\| = |\mathbf{a}^{\top} \mathbf{b}'  | +  \sqrt{2(1 - \mathbf{b}^{\top} \mathbf{b}')}$.
\end{proof}

\begin{proposition} 
Suppose $\mathbf{K}_1$ and $\mathbf{K}_2$ are $d \times d$ positive semi-definite matrices with the constraint that each of its diagonals is $1$. Then $\operatorname{I}_{\alpha}(\mathbf{K}_1; \mathbf{K}_2) \leq \log d$.
\end{proposition}

\begin{proof}
The proof is straightforward by using the inequalities introduced in \cite{giraldo2014measures} as follows. $\operatorname{I}_{\alpha}(\mathbf{K}_1; \mathbf{K}_2) = \operatorname{H}_{\alpha}(\mathbf{K}_1) + \operatorname{H}_{\alpha}(\mathbf{K}_2) - \operatorname{H}_{\alpha}(\mathbf{K}_1 \odot \mathbf{K}_2) \leq \operatorname{H}_{\alpha}(\mathbf{K}_1) \leq  \log d$.   
\end{proof}

\begin{theorem} 
When $\alpha>0$, {Barlow Twins and spectral contrastive learning losses maximize the matrix mutual information when their loss value is $0$.} 
\end{theorem}

\begin{proof}

Denote the (batch normalized) vectors for each dimension $i$ ($1 \leq i \leq d$) of the online and target networks as $\bar{\mathbf{z}}^{(1)}_i$ and $\bar{\mathbf{z}}^{(2)}_i$. 

Take $\mathbf{K}_1 = [\bar {\mathbf{z}}^{(1)}_1 \cdots \bar{\mathbf{z}}^{(1)}_d]^{\top}[\bar {\mathbf{z}}^{(1)}_1 \cdots \bar{\mathbf{z}}^{(1)}_d]$ and $\mathbf{K}_2 = [\bar {\mathbf{z}}^{(2)}_1 \cdots \bar{\mathbf{z}}^{(2)}_d]^{\top}[\bar {\mathbf{z}}^{(2)}_1 \cdots \bar{\mathbf{z}}^{(2)}_d]$.

{When the loss value is 0}, Barlow Twins loss has $\bar{\mathbf{z}}^{(1)}_i = \bar{\mathbf{z}}^{(2)}_i$ for each $i \in \{1, \cdots, d\}$ and $(\bar{\mathbf{z}}^{(1)}_i)^{\top} \bar{\mathbf{z}}^{(2)}_j = 0$ for each $i \neq j$. Then for each $i \neq j$, $(\bar {\mathbf{z}}^{(1)}_i)^{\top} \bar {\mathbf{z}}^{(1)}_j = (\bar {\mathbf{z}}^{(1)}_i)^{\top} \bar {\mathbf{z}}^{(2)}_j = 0$. Similarly, $(\bar{\mathbf{z}}^{(2)}_i)^{\top} \bar {\mathbf{z}}^{(2)}_j = 0$. Then $\mathbf{K}_1 = \mathbf{K}_2 =\mathbf{I}_d$. By noticing $\operatorname{H}_{\alpha}(\mathbf{I}_d, \mathbf{I}_d) = \log d$. {Then the matrix mutual information is maximized.}

When performing spectral contrastive learning, the loss is $\sum^B_{i=1} \mid \mid \mathbf{z}^{(1)}_i - \mathbf{z}^{(2)}_i \mid \mid^{2}_2 + \lambda \sum_{ i \neq j} ((\mathbf{z}^{(1)}_i)^{\top} \mathbf{z}^{(2)}_j)^2$. {Take $\mathbf{K}_1 = \mathbf{Z}^{\top}_1\mathbf{Z}_1$ and $\mathbf{K}_2=\mathbf{Z}^{\top}_2\mathbf{Z}_2$,} the results follows similarly. Thus concludes the proof.
\end{proof}

\begin{proposition} 
The joint entropy lower bounds the representation rank in two branches by having the inequality as follows:
\begin{align*}
\operatorname{H}_1(\mathbf{K}_1, \mathbf{K}_2) 
\leq \log(\operatorname{rank}(\mathbf{K}_1 \odot \mathbf{K}_2))  \leq \log \operatorname{rank}(\mathbf{K}_1) + \log \operatorname{rank}(\mathbf{K}_2).  
\end{align*}
\begin{align*}
\max \{\operatorname{H}_{\alpha}(\mathbf{K}_1), \operatorname{H}_{\alpha}(\mathbf{K}_2) \} 
 \leq  \operatorname{H}_{\alpha}(\mathbf{K}_1, \mathbf{K}_2)  \leq \operatorname{H}_{\alpha}(\mathbf{K}_1) + \operatorname{H}_{\alpha}(\mathbf{K}_2).  
\end{align*}
\end{proposition}

\begin{proof}
The first inequality comes from the fact that effective rank lower bounds the rank. The second inequality comes from the rank inequality of Hadamard product. The third and fourth inequalities follow from \cite{giraldo2014measures}.  
\end{proof}

\begin{proposition} 
Suppose $\mathbf{K}$ is a $d \times d$ matrix with the constraint that each of its diagonals is $1$. Then the following equalities holds:
\begin{equation}
\begin{aligned}
&\operatorname{H}_1(\mathbf{K}) = \log d - \frac{1}{d} \operatorname{KL}(\mathbf{K}, \mathbf{I}_d),\\
&
\operatorname{TCR}_{\mu}(\mathbf{K}) = d \log(1+\mu) -\operatorname{KL}(\mathbf{I}_d, \frac{1}{1+\mu} (\mu \mathbf{I}_d + \mathbf{K})).
\end{aligned}
\end{equation}
\end{proposition}

\begin{proof}
The proof is from directly using the definition of matrix KL divergence.    
\end{proof}

\begin{proposition} 
The (joint) total coding rate upperbounds the rate in two branches by having the inequality as follows:
\begin{equation}
\operatorname{TCR}_{\mu^2+2 \mu}(\mathbf{K}_1 \odot \mathbf{K}_2) \geq \operatorname{TCR}_{\mu}(\mathbf{K}_1) + \operatorname{TCR}_{\mu}(\mathbf{K}_2).   
\end{equation}
\end{proposition}

\begin{proof}
The inequality comes from the determinant inequality of Hadamard products and the fact that $(\mathbf{K}_1 +\mu \mathbf{I}) \odot (\mathbf{K}_1 +\mu \mathbf{I}) = \mathbf{K}_1 \odot \mathbf{K}_2 + (\mu^2+2\mu)\mathbf{I}$.       
\end{proof}

\begin{theorem}
1. In the spectral contrastive loss, we have 
\begin{equation*}
\operatorname{H}_2(\mathbf{Z}^{\top}_1\mathbf{Z}_1, \mathbf{Z}^{\top}_2\mathbf{Z}_2) \geq \log B -  \log(1+ (2 + \frac{2}{B \lambda}) \mathcal{L}_{SC} ).
\end{equation*}

2. In the Barlow Twins loss, we have 
\begin{equation*}
\operatorname{H}_2(\bar{\mathbf{Z}}_1\bar{\mathbf{Z}}^{\top}_1, \bar{\mathbf{Z}}_2\bar{\mathbf{Z}}^{\top}_2) \geq \log d -  \log(1+ \frac{2}{d \lambda} \mathcal{L}_{BT} + 4 \sqrt{d\mathcal{L}_{BT}} ).
\end{equation*}

\end{theorem}

\begin{proof}
We will only present the proof for spectral contrastive loss as Barlow Twins loss is similar.

By Proposition \ref{joint entropy reduction}, we know that $\operatorname{H}_2(\mathbf{Z}^{\top}_1\mathbf{Z}_1, \mathbf{Z}^{\top}_2\mathbf{Z}_2) = 2\log B - \log {|| \mathbf{Z}^{\top}_1\mathbf{Z}_1 \odot \mathbf{Z}^{\top}_2\mathbf{Z}_2 ||^2_F} = \log B - \log \frac{|| \mathbf{Z}^{\top}_1\mathbf{Z}_1 \odot \mathbf{Z}^{\top}_2\mathbf{Z}_2 ||^2_F}{B}.$

On the other hand, $\frac{|| \mathbf{Z}^{\top}_1\mathbf{Z}_1 \odot \mathbf{Z}^{\top}_2\mathbf{Z}_2 ||^2_F}{B} = 1 + \frac{\sum_{i \neq j} ((\mathbf{z}^{(1)}_i)^{\top}  \mathbf{z}^{(1)}_j)^2 (\mathbf{z}^{(2)}_i)^{\top}  \mathbf{z}^{(2)}_j)^2 }{B}.$

Using Lemma \ref{cross and cov}, we know that $((\mathbf{z}^{(1)}_i)^{\top}  \mathbf{z}^{(1)}_j)^2 \leq (|(\mathbf{z}^{(1)}_i)^{\top}  \mathbf{z}^{(2)}_j| +  \| \mathbf{z}^{(1)}_j - \mathbf{z}^{(2)}_j \| )^2 \leq 2(|(\mathbf{z}^{(1)}_i)^{\top}  \mathbf{z}^{(2)}_j|^2 +  \| \mathbf{z}^{(1)}_j - \mathbf{z}^{(2)}_j \|^2).$

Therefore, $\sum_{i \neq j} ((\mathbf{z}^{(1)}_i)^{\top}  \mathbf{z}^{(1)}_j)^2 (\mathbf{z}^{(2)}_i)^{\top}  \mathbf{z}^{(2)}_j)^2 \leq  \sum_{i \neq j} ((\mathbf{z}^{(1)}_i)^{\top}  \mathbf{z}^{(1)}_j)^2 \leq 2(\frac{1}{\lambda} + B) \mathcal{L}_{SC}$.

Combining all the above, the conclusion follows.

\end{proof}

\begin{proposition} 
\begin{equation}
\mathbf{I}_d = \operatorname{argmax} \operatorname{H}_{\alpha}(\mathbf{K}), \operatorname{and} 
\mathbf{I}_d = \operatorname{argmax} 
\operatorname{TCR}_{\mu}(\mathbf{K}),
\end{equation} where the maximization is over $d \times d$ positive semi-definite matrices with the constraint that each of its diagonals is $1$.
\end{proposition}
\begin{proof}
Specifically, matrix entropy is Shannon entropy on the spectrum and the uniform distribution on the spectrum maximizes the entropy. Consider the spectrum will also give the result for TCR. Another proof directly using matrix KL divergence can be seen in \citep{zhang2023kernel}.
\end{proof}

\begin{theorem} 
When $\alpha>0$, {Barlow twins and spectral contrastive learning losses maximize the matrix joint entropy when their loss value is $0$.} 
\end{theorem}

\begin{proof}
Denote the (along batch normalized) vectors for each dimension $i$ ($1 \leq i \leq d$) of the online and target networks as $\bar{\mathbf{z}}^{(1)}_i$ and $\bar{\mathbf{z}}^{(2)}_i$. Take $\mathbf{K}_1 = [\bar{\mathbf{z}}^{(1)}_1 \cdots \bar{\mathbf{z}}^{(1)}_d]^{\top}[\bar {\mathbf{z}}^{(1)}_1 \cdots \bar {\mathbf{z}}^{(1)}_d]$ and $\mathbf{K}_2 = [\bar{\mathbf{z}}^{(2)}_1 \cdots \bar{\mathbf{z}}^{(2)}_d]^{\top}[\bar {\mathbf{z}}^{(2)}_1 \cdots \bar {\mathbf{z}}^{(2)}_d]$. {When the loss value is 0,} Barlow Twins loss has $\bar{\mathbf{z}}^{(1)}_i = \bar{\mathbf{z}}^{(2)}_i$ for each $i \in \{1, \cdots, d\}$ and $(\bar{\mathbf{z}}^{(1)}_i)^{\top} \bar{\mathbf{z}}^{(2)}_j = 0$ for each $i \neq j$. Then for each $i \neq j$, $(\bar {\mathbf{z}}^{(1)}_i)^{\top} \bar {\mathbf{z}}^{(1)}_j = (\bar {\mathbf{z}}^{(1)}_i)^{\top} \bar {\mathbf{z}}^{(2)}_j = 0$. Similarly, $(\bar{\mathbf{z}}^{(2)}_i)^{\top} \bar {\mathbf{z}}^{(2)}_j = 0$. Then $\mathbf{K}_1 = \mathbf{K}_2 =\mathbf{I}_d$. By noticing $\mathbf{I}_d \odot \mathbf{I}_d = \mathbf{I}_d$. {Then the matrix joint entropy is maximized by noticing Proposition \ref{entropy optimal point}.}

When performing spectral contrastive learning, the loss is $\sum^B_{i=1} \mid \mid \mathbf{z}^{(1)}_i - \mathbf{z}^{(2)}_i \mid \mid^{2}_2 + \lambda \sum_{ i \neq j} ((\mathbf{z}^{(1)}_i)^{\top} \mathbf{z}^{(2)}_j)^2$. {Take $\mathbf{K}_1 = \mathbf{Z}^{\top}_1\mathbf{Z}_1$ and $\mathbf{K}_2=\mathbf{Z}^{\top}_2\mathbf{Z}_2$}, the results follows similarly. 
\end{proof}

\begin{theorem}
$\operatorname{TCR}_{\mu}(\bar{\mathbf{Z}}_1\bar{\mathbf{Z}}^{\top}_1\odot \bar{\mathbf{Z}}_2\bar{\mathbf{Z}}^{\top}_2) \geq d \log(\mu +1) - \frac{1}{2 \mu^2} ( \frac{2}{\lambda} \mathcal{L}_{BT} + 4(d-1) \sqrt{d\mathcal{L}_{BT}} ).$   
\end{theorem}
\begin{proof}
\begin{lemma}
$\forall$ $x,a \geq 0$, we have $\log(1+x) \geq \log(1+a) -\frac{1}{2}(x-a)^2+\frac{1}{1+a}(x-a)$.
\end{lemma}
\begin{proof}
The proof of the lemma is direct from taking the derivative and finding that $x=a$ is the minimal point.
\end{proof}
Denote $\lambda_i$ as the eigenvalues of $\bar{\mathbf{Z}}_1\bar{\mathbf{Z}}^{\top}_1\odot \bar{\mathbf{Z}}_2\bar{\mathbf{Z}}^{\top}_2$, we know that $\lambda_i \geq 0$ and $\sum^d_{i=1} \lambda_i = d$ and $\sum^d_{i=1} \lambda^2_i = \| \bar{\mathbf{Z}}_1\bar{\mathbf{Z}}^{\top}_1\odot \bar{\mathbf{Z}}_2\bar{\mathbf{Z}}^{\top}_2\|^2_F$. Then take $a=\frac{1}{\mu}$ in the above lemma, we will get the following: $\operatorname{TCR}_{\mu}(\bar{\mathbf{Z}}_1\bar{\mathbf{Z}}^{\top}_1\odot \bar{\mathbf{Z}}_2\bar{\mathbf{Z}}^{\top}_2) = \log \operatorname{det}(\mu \mathbf{I}_d + \bar{\mathbf{Z}}_1\bar{\mathbf{Z}}^{\top}_1\odot \bar{\mathbf{Z}}_2\bar{\mathbf{Z}}^{\top}_2) = \sum^d_{i=1} \log(\mu + \lambda_i) = d \log(\mu) + \sum^d_{i=1} \log(1 + \frac{\lambda_i}{\mu}) \geq d \log(\mu) + \sum^d_{i=1}(\log(1 + \frac{1}{\mu}) + \frac{1}{1+\frac{1}{\mu}}(\frac{\lambda_i}{\mu} - \frac{1}{\mu}) - \frac{1}{2} (\frac{\lambda_i}{\mu} - \frac{1}{\mu})^2) = d \log(\mu +1) + \frac{1}{2\mu^2}d - \frac{1}{2 \mu^2} \sum^d_{i=1} \lambda^2_i = d \log(\mu +1) + \frac{1}{2\mu^2}d - \frac{1}{2 \mu^2} \| \bar{\mathbf{Z}}_1\bar{\mathbf{Z}}^{\top}_1\odot \bar{\mathbf{Z}}_2\bar{\mathbf{Z}}^{\top}_2  \|^2_F$. If we denote $\mathbf{K}_1=\bar{\mathbf{Z}}_1\bar{\mathbf{Z}}^{\top}_1$ and $\mathbf{K}_2=\bar{\mathbf{Z}}_2\bar{\mathbf{Z}}^{\top}_2$, then $\operatorname{TCR}_{\mu} (\mathbf{K}_1 \odot \mathbf{K}_2) \geq d \log(\mu +1) - \frac{1}{2\mu^2} \sum_{i \neq j} \mathbf{K}^2_1(i,j)\mathbf{K}^2_2(i,j) \geq d \log(\mu +1) - \frac{1}{2\mu^2} \sum_{i \neq j} \mathbf{K}^2_1(i,j) \geq d \log(\mu +1) - \frac{1}{2\mu^2} \sum_{i \neq j} 2(\mathcal{C}^2_{i,j} + (2-2\mathcal{C}_{j,j})) \geq d \log(\mu +1) -\frac{1}{2 \mu^2} (\frac{2}{\lambda} \mathcal{L}_{BT} + 4(d-1) \sqrt{d \mathcal{L}_{BT}})$.
\end{proof}

{\textbf{Remark}: Following the proof of our Theorem \ref{MI max} and Theorem \ref{thm:joint-entropy 1} and Theorem \ref{TCR bound}, our theoretical results can be generalized to sample contrastive and dimension contrastive methods defined in \citep{garrido2022duality}. As pointed out by \citep{garrido2022duality}, sample and dimension contrastive methods contain many famous self-supervised methods (Proposition 3.2 of \citep{garrido2022duality}).}

Below are other ways of proving the case of $\alpha=2$.

\begin{theorem} 
When $\alpha=2$, {Barlow Twins and spectral contrastive learning losses maximize the matrix mutual information when their loss value is $0$.} 
\end{theorem}

We shall first present a lemma as follows:
\begin{lemma} \label{reduction lemma}
Given two positive integers $n,m$. Denote two sequences $\mathbf{x}=(x_1, \cdots, x_m)$ and $\mathbf{y}=(y_1, \cdots, y_m)$. Then $\mathbf{x}=\mathbf{y}=0$ is the unique solution to the following optimization problem: 
\begin{equation*}
\operatorname{min}_{0 \leq \mathbf{x} \leq 1, 0 \leq \mathbf{y} \leq 1} \frac{(n + \sum^{m}_{i=1} x_i)(n + \sum^{m}_{i=1} y_i)}{n + \sum^{m}_{i=1} x_i y_i}.   
\end{equation*}
\end{lemma}

\begin{proof}
\label{proof:reduction}
Notice that 
$$
\frac{(n + \sum^{m}_{i=1} x_i)(n + \sum^{m}_{i=1} y_i)}{n + \sum^{m}_{i=1} x_i y_i} - n = \frac{n(\sum^{m}_{i=1}x_i + \sum^{m}_{i=1} y_i)-n \sum^{m}_{i=1} x_i y_i + (\sum^{m}_{i=1} x_i)(\sum^{m}_{i=1} y_i)}{n + \sum^{m}_{i=1} x_i y_i}
$$

Note $x_i \geq x^2_i$ and $y_i \geq y^2_i$. Then we shall get inequality as follows:
\begin{equation*}
\sum^{m}_{i=1} x_i + \sum^{m}_{i=1} y_i \geq 2 \sqrt{(\sum^{m}_{i=1} x_i)(\sum^{m}_{i=1} y_i)}  \geq 2 \sqrt{(\sum^{m}_{i=1} x^2_i)(\sum^{m}_{i=1} y^2_i)} \geq 2  \sum^{m}_{i=1} x_i y_i. 
\end{equation*}
Thus the above optimization problem gets a minimum of $n$, with $\mathbf{x}=\mathbf{y}=0$ the unique solution.
\end{proof}

\begin{proof}

Denote the (batch normalized) vectors for each dimension $i$ ($1 \leq i \leq d$) of the online and target networks as $\bar{\mathbf{z}}^{(1)}_i$ and $\bar{\mathbf{z}}^{(2)}_i$. 

Take $\mathbf{K}_1 = [\bar {\mathbf{z}}^{(1)}_1 \cdots \bar{\mathbf{z}}^{(1)}_d]^{\top}[\bar {\mathbf{z}}^{(1)}_1 \cdots \bar{\mathbf{z}}^{(1)}_d]$ and $\mathbf{K}_2 = [\bar {\mathbf{z}}^{(2)}_1 \cdots \bar{\mathbf{z}}^{(2)}_d]^{\top}[\bar {\mathbf{z}}^{(2)}_1 \cdots \bar{\mathbf{z}}^{(2)}_d]$. 

From Proposition \ref{mutual information reduction}, it is clear that the mutual information $\mathbf{I}_2(\mathbf{K}_1; \mathbf{K}_2)$ is maximized iff $\frac{|| \mathbf{K}_1 ||^2_F || \mathbf{K}_2 ||^2_F}{|| \mathbf{K}_1 \odot \mathbf{K}_2 ||^2_F}$ is minimized. Take $((\bar{\mathbf{z}}^{(1)}_i)^{\top}\bar{\mathbf{z}}^{(1)}_j)^2$ and $((\bar{\mathbf{z}}^{(2)}_i)^{\top}\bar{ \mathbf{z}}^{(2)}_j)^2$ as elements of $\mathbf{x}$ and $\mathbf{y}$ in Lemma \ref{reduction lemma}, then we can see the maximal mutual information is attained iff $(\bar{\mathbf{z}}^{(1)}_i)^{\top} \bar {\mathbf{z}}^{(1)}_j = 0$ and $(\bar{\mathbf{z}}^{(2)}_i)^{\top} \bar {\mathbf{z}}^{(2)}_j = 0$. 

{When the loss value is 0}, Barlow Twins loss has $\bar{\mathbf{z}}^{(1)}_i = \bar{\mathbf{z}}^{(2)}_i$ for each $i \in \{1, \cdots, d\}$ and $(\bar{\mathbf{z}}^{(1)}_i)^{\top} \bar{\mathbf{z}}^{(2)}_j = 0$ for each $i \neq j$. Then for each $i \neq j$, $(\bar {\mathbf{z}}^{(1)}_i)^{\top} \bar {\mathbf{z}}^{(1)}_j = (\bar {\mathbf{z}}^{(1)}_i)^{\top} \bar {\mathbf{z}}^{(2)}_j = 0$. Similarly, $(\bar{\mathbf{z}}^{(2)}_i)^{\top} \bar {\mathbf{z}}^{(2)}_j = 0$. {Then the matrix mutual information is maximized.}

When performing spectral contrastive learning, the loss is $\sum^B_{i=1} \mid \mid \mathbf{z}^{(1)}_i - \mathbf{z}^{(2)}_i \mid \mid^{2}_2 + \lambda \sum_{ i \neq j} ((\mathbf{z}^{(1)}_i)^{\top} \mathbf{z}^{(2)}_j)^2$. {Take $\mathbf{K}_1 = \mathbf{Z}^{\top}_1\mathbf{Z}_1$ and $\mathbf{K}_2=\mathbf{Z}^{\top}_2\mathbf{Z}_2$,} the results follows similarly. Thus concludes the proof.
\end{proof}

\begin{theorem} 
When $\alpha=2$, {Barlow Twins and spectral contrastive learning losses maximize the matrix joint entropy when their loss value is $0$.} 
\end{theorem}

\begin{proof}
\label{proof:joint-entropy}
Denote the (along batch normalized) vectors for each dimension $i$ ($1 \leq i \leq d$) of the online and target networks as $\bar{\mathbf{z}}^{(1)}_i$ and $\bar{\mathbf{z}}^{(2)}_i$. Take $\mathbf{K}_1 = [\bar{\mathbf{z}}^{(1)}_1 \cdots \bar{\mathbf{z}}^{(1)}_d]^{\top}[\bar {\mathbf{z}}^{(1)}_1 \cdots \bar {\mathbf{z}}^{(1)}_d]$ and $\mathbf{K}_2 = [\bar{\mathbf{z}}^{(2)}_1 \cdots \bar{\mathbf{z}}^{(2)}_d]^{\top}[\bar {\mathbf{z}}^{(2)}_1 \cdots \bar {\mathbf{z}}^{(2)}_d]$. From Proposition \ref{joint entropy reduction}, it is clear that the joint entropy $\operatorname{H}_2(\mathbf{K}_1, \mathbf{K}_2)$ is maximized iff $\mid \mid \mathbf{K}_1 \odot \mathbf{K}_2 \mid \mid^2_F$ is minimized. Note from the definition of Frobenius norm, $\mid \mid \mathbf{K}_1 \odot \mathbf{K}_2 \mid \mid^2_F= \sum_{i,j} ((\mathbf{K}_1 \odot \mathbf{K}_2)(i, j))^2 = \sum_{i,j} (\mathbf{K}_1(i, j) \mathbf{K}_2(i, j))^2$. {When the loss value is 0,} Barlow Twins loss has $\bar{\mathbf{z}}^{(1)}_i = \bar{\mathbf{z}}^{(2)}_i$ for each $i \in \{1, \cdots, d\}$ and $(\bar{\mathbf{z}}^{(1)}_i)^{\top} \bar{\mathbf{z}}^{(2)}_j = 0$ for each $i \neq j$. Then for each $i \neq j$, $(\bar {\mathbf{z}}^{(1)}_i)^{\top} \bar {\mathbf{z}}^{(1)}_j = (\bar {\mathbf{z}}^{(1)}_i)^{\top} \bar {\mathbf{z}}^{(2)}_j = 0$. Similarly, $(\bar{\mathbf{z}}^{(2)}_i)^{\top} \bar {\mathbf{z}}^{(2)}_j = 0$.
When performing spectral contrastive learning, the loss is $\sum^B_{i=1} \mid \mid \mathbf{z}^{(1)}_i - \mathbf{z}^{(2)}_i \mid \mid^{2}_2 + \lambda \sum_{ i \neq j} ((\mathbf{z}^{(1)}_i)^{\top} \mathbf{z}^{(2)}_j)^2$. {Take $\mathbf{K}_1 = \mathbf{Z}^{\top}_1\mathbf{Z}_1$ and $\mathbf{K}_2=\mathbf{Z}^{\top}_2\mathbf{Z}_2$}, the results follows similarly. 
\end{proof}

\section{Ablation study}

\begin{table}[htb]
\centering
\caption{Linear probing accuracy (\%) of M-MAE for ViT-Base with varying \( \mu \) coefficients.}
\vspace{1ex}
\begin{tabular}{cccccccc}
\toprule
\( \mu \) Coefficient & 0.1 & 0.5 & 0.75 & 1 & 1.25 & 1.5 & 3 \\
\midrule
Accuracy & 58.61 & 59.38 & 59.87 & \textbf{62.40} & 59.54 & 57.76 & 50.46 \\
\bottomrule
\end{tabular}
\vspace{-3.2mm}
\label{table:ablation}
\end{table}

To investigate the robustness of our approach to variations in hyperparameters, we perform an ablation study focusing on the coefficients $\mu$ in the TCR loss. The results for different $\mu$ values are summarized as in Table~\ref{table:ablation}.

As observed in Table~\ref{table:ablation}, the M-MAE model exhibits a peak performance at $\mu = 1$ for ViT-Base. Deviating from this value leads to a gradual degradation in performance, illustrating the importance of careful hyperparameter tuning for maximizing the benefits of the TCR loss.

\section{More experiments}\label{more exp}

\subsection{The tendency under different temperatures}

One of the important hyper-parameters in SimCLR is the temperature in InfoNCE loss. We plot the matrix mutual information and matrix joint entropy during the pretraining of SimCLR on CIFAR-10 with different temperatures. We set temperatures as 0.3, 0.5, 0.7. From the Figure \ref{matrix information and temp}, we can observe that the increase of matrix mutual information or matrix joint entropy during training ties closely with the final accuracy. As temperature = 0.3 outperforms 0.5 and 0.7 in KNN accuracy, it also has the biggest matrix mutual information and matrix joint entropy value.

\begin{figure}[htb]
\centering
\begin{subfigure}[b]{0.49\textwidth}
\includegraphics[width=\linewidth]{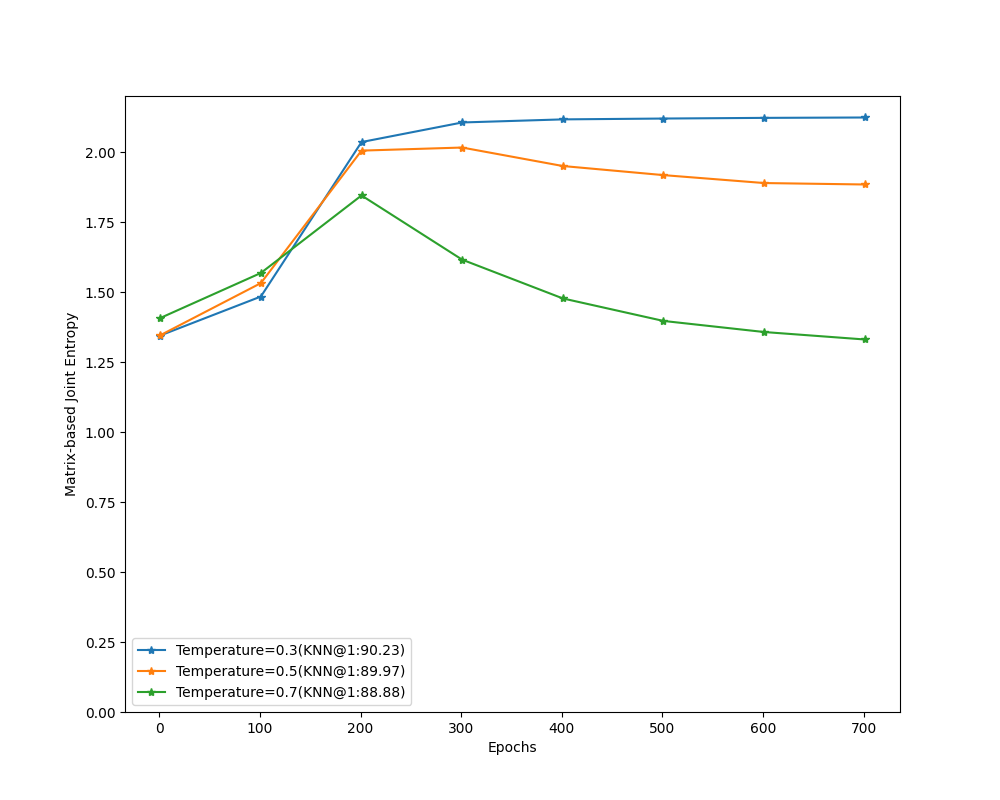}
\caption{Matrix joint entropy. }
\label{matrix joint entropy and temp}
\end{subfigure}
\begin{subfigure}[b]{0.49\textwidth}
\includegraphics[width=\linewidth]{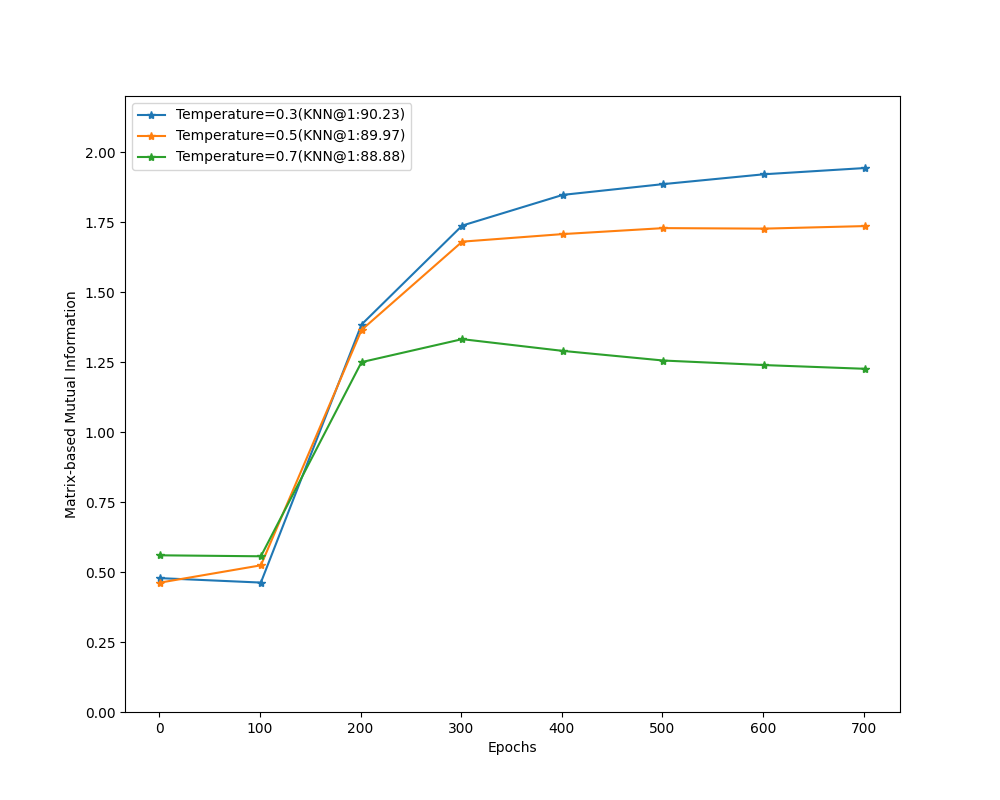}
\caption{Matrix mutual information }
\label{matrix mutual information and temp}
\end{subfigure}
\caption{Tendency of matrix information quantities under different temperatures. The experiments are conducted on CIFAR-10 using SimCLR.}
\label{matrix information and temp}
\end{figure}

\subsection{Longer training on masked modeling}

We have conducted experiments on CIFAR-100. The hyper-parameters are similar to U-MAE, and we set $\mu=1$ and pretrain CIFAR-100 for 1000 epochs. M-MAE may use hyperparameters that are not identical to U-MAE to fully reflect its potential. However, due to time constraints, we were unable to extensively search for these hyperparameters. We believe that with more reasonable hyperparameters, M-MAE can achieve even better results. As shown in the Table \ref{tab:my_table}, our method performs remarkably well, even without an exhaustive hyperparameter search.

\begin{table}[ht]
\centering
\caption{Results on CIFAR-100 under various masked modeling pretraining algorithms.}
\label{tab:my_table}
\begin{tabular}{|c|c|c|c|c|}
\hline
Method & linear probe@1 &	linear probe@5 & finetune@1 & finetune@5 \\
\hline
M-MAE (vit-base) & \textbf{60.9} &	\textbf{88.1}&	83.8&	\textbf{97.0} \\
\hline
U-MAE (vit-base) & 59.7&	86.8&	84.1&	96.5\\
\hline
MAE (vit-base) & 59.2&	86.8&	\textbf{84.5}&	96.6 \\
\hline
\end{tabular}
\end{table}

We plot the effective rank of learned representations under algorithms MAE, U-MAE, and M-MAE in Figure \ref{fig:my_image}. We find that M-MAE has the biggest effective rank among the algorithms U-MAE has its effective rank bigger than MAE, and the effective ranks of M-MAE have an increasing trend during training. This aligns with our theorem which shows U-MAE can be seen as a second-order approximation of our M-MAE method.

\begin{figure}[ht]
\centering
\includegraphics[width=0.5\textwidth]{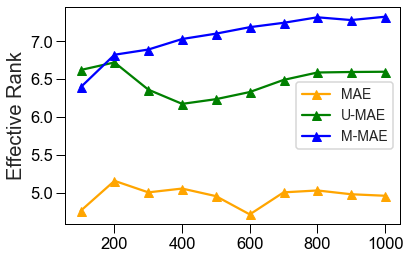}
\caption{The effective rank during pre-training.}
\label{fig:my_image}
\end{figure}

\subsection{Measuring the difference between Siamese branches}

As we have discussed the total or shared information in the Siamese architectures, we haven't used the matrix information-theoretic tools to analyze the \textbf{differences} in the two branches. 

From information theory, we know that KL divergence is a special case of $f$-divergence defined as follows:
\begin{definition}
For two probability distributions $\mathbf{P}$ and $\mathbf{Q}$, where $\mathbf{P}$ is absolutely continuous with respect to $\mathbf{Q}$. Suppose $\mathbf{P}$ and $\mathbf{Q}$ has density $p(x)$ and $q(x)$ respectively. Then for a convex function $f$ is defined on non-negative numbers which is right-continuous at $0$ and satisfies $f(1)=0$. The $f$-divergence is defined as: 
\begin{equation}
D_f(\mathbf{P} \mid \mid \mathbf{Q}) = \int f\left(\frac{p(x)}{q(x)}\right) q(x) dx.    
\end{equation}    
\end{definition}
When $f(x)=x \log x$ will recover the KL divergence. Then a natural question arises: are there other $f$ divergences that can be easily generalized to matrices? Note by taking $f(x)= -(x+1) \log \frac{x+1}{2} + x \log x$, we shall retrieve JS divergence. Recently, \cite{hoyos2023representation} generalized JS divergence to the matrix regime.
\begin{definition}[Matrix JS divergence \citep{hoyos2023representation}] Suppose matrix $\mathbf{K}_1, \mathbf{K}_2 \in \mathbb{R}^{n \times n}$ which $\mathbf{K}_1(i,i)=\mathbf{K}_2(i,i)=1$ for every $i=1, \cdots, n$. The  Jensen-Shannon (JS) divergence between these two matrices $\mathbf{K}_1$ and $\mathbf{K}_2$ is defined as
\begin{equation*}
\operatorname{JS} \left(\mathbf{K}_1 \mid \mid \mathbf{K}_2 \right)= \mathbf{H}_1\left (\frac{\mathbf{K}_1 + \mathbf{K}_2}{2}\right ) - \frac{\mathbf{H}_1(\mathbf{K}_1) + \mathbf{H}_1(\mathbf{K}_2)}{2}.
\end{equation*}
\end{definition}

One may think the matrix KL divergence is a good candidate, but this quantity has some severe problems making it not a good choice. One problem is that the matrix KL divergence is not symmetric. Another problem is that the matrix KL divergence is not bounded, and sometimes may even be undefined. Recall these drawbacks are similar to that of KL divergence in traditional information theory. In traditional information theory, JS divergence successfully overcomes these drawbacks, thus we may use the matrix JS divergence to measure the differences between branches. As matrix JS divergence considers the interactions between branches, we shall also include the JS divergence between eigenspace distributions as another difference measure.

Specifically, the online and target batch normalized feature correlation matrices can be calculated by $\mathbf{K}_1=\bar{\mathbf{Z}}_1\bar{\mathbf{Z}}^{\top}_1$ and $\mathbf{K}_2=\bar{\mathbf{Z}}_2\bar{\mathbf{Z}}^{\top}_2$. Denote $\mathbf{p}_1$ and $\mathbf{p}_2$ the online and target (normalized) eigen distribution respectively. We plot the matrix JS divergence $\operatorname{JS}(\mathbf{K_1}, \mathbf{K_2})$ between branches in Figure~\ref{fig:MJS}. It is evident that throughout the whole training, the JS divergence is a small value, indicating a small gap between the branches. More interestingly, the JS divergence increases during training, which means that an effect of ``symmetry-breaking" may exist in self-supervised learning. Additionally, we plot the plain JS divergence $\operatorname{JS}(\mathbf{p_1}, \mathbf{p_2})$ between branches in Figure~\ref{fig:PJS}. It is evident that $\operatorname{JS}(\mathbf{p_1}, \mathbf{p_2})$ is very small, even compared to $\operatorname{JS}(\mathbf{K_1}, \mathbf{K_2})$. Thus we hypothesize that the ``symmetry-breaking" phenomenon is mainly due to the interactions between Siamese branches.

\begin{figure}[htb]
\centering
\begin{subfigure}[b]{0.45\textwidth}
\includegraphics[width=\linewidth]{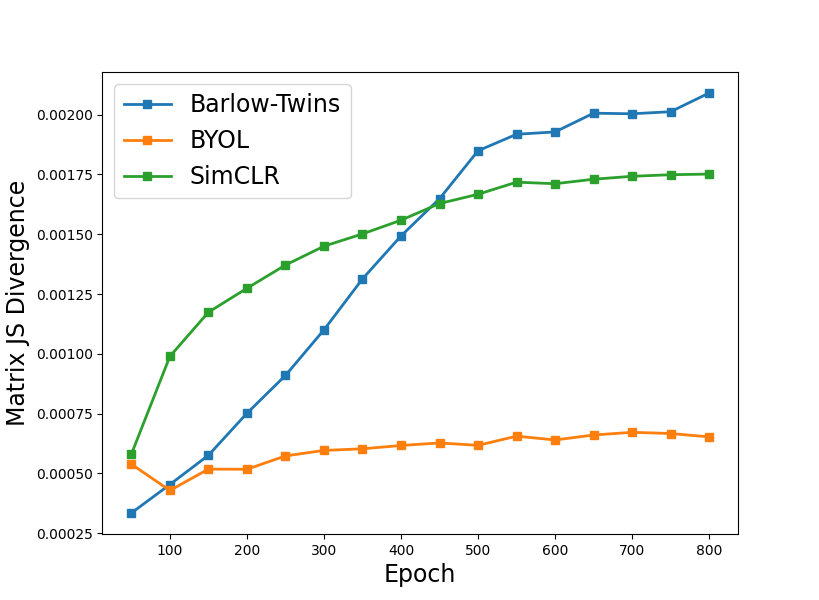}
\caption{Matrix JS Divergence.}
\label{fig:MJS}
\end{subfigure}
\begin{subfigure}[b]{0.45\textwidth}
\includegraphics[width=\linewidth]{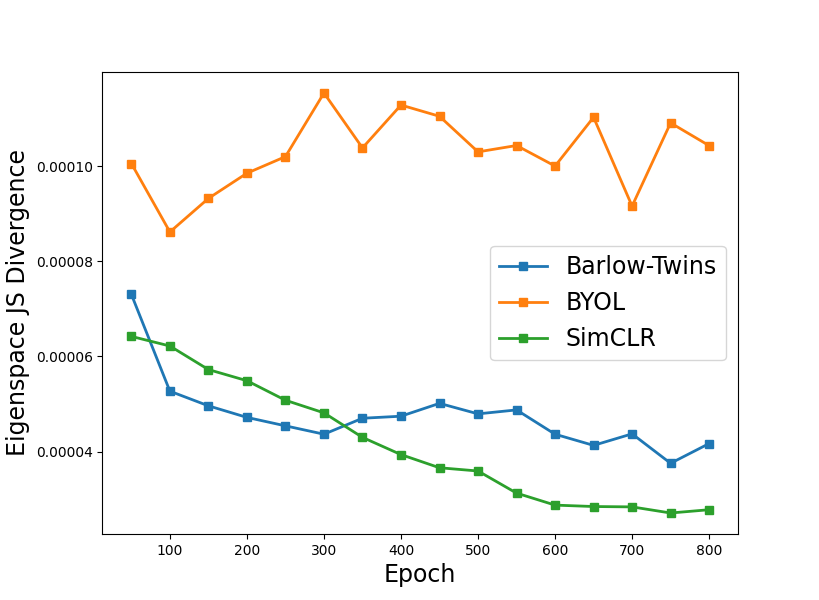}
\caption{Eigenspace JS Divergence.}
\label{fig:PJS}
\end{subfigure}
\caption{Visualization of matrix JS divergence and eigenspace JS divergence on CIFAR10 for Barlow-Twins, BYOL, and SimCLR.}
\label{JS Divergence}
\end{figure}

\end{document}